\documentclass[envcountsect,envcountsame,final,orivec]{llncs}

\usepackage{xcolor}
\usepackage{times}
\usepackage{soul}
\usepackage{url}
\usepackage[citecolor=blue]{hyperref}
\hypersetup{
    colorlinks,
    linkcolor={green!65!black},
    citecolor={blue!70!black},
    urlcolor={blue!80!black}
}

\usepackage{multirow}
\usepackage[utf8]{inputenc}
\usepackage[small]{caption}
\usepackage{graphicx}
\usepackage{amsmath}
\usepackage{booktabs}
\usepackage{algorithm}
\usepackage{subcaption}
\usepackage{colortbl}
\usepackage{tabularx}
\usepackage{graphbox}

\usepackage[left=3cm,right=3cm,top=3.8cm,bottom=3.8cm]{geometry} 

\DeclareMathAlphabet{\mathbfit}{OML}{cmm}{b}{it}

\usepackage{amssymb,amsfonts,amsxtra,mathrsfs}
\usepackage{color}
\usepackage{stmaryrd}
\usepackage{listings}
\usepackage{esvect}

\makeatletter
\newcommand{\TODO}{\@ifnextchar[{\@TODO}{\@@TODO}}
\long\def\@TODO[#1]#2{\bgroup\color{red}\ifmmode\else\itshape\fi TODO[#1]: #2\egroup}
\long\def\@@TODO#1{\bgroup\color{red}\ifmmode\else\itshape\fi TODO: #1\egroup}
\makeatother

\usepackage{thmtools}
\usepackage{thm-restate}

\spnewtheorem{examplebf}[theorem]{Example}{\bfseries}{\rmfamily}

\newcommand{\overbar}[1]{\mkern 1.5mu\overline{\mkern-1.5mu#1\mkern-1.5mu}\mkern 1.5mu}

\newcommand{\R}{\ensuremath{\mathbb{R}}}

\newcommand{\F}{\ensuremath{\mathbb{F}}}
\newcommand{\N}{\ensuremath{\mathbb{N}}}

\DeclareMathOperator*{\argmax}{argmax} 

\DeclareMathOperator{\votes}{votes}
\def\ok#1{\mbox{\raisebox{0ex}[1ex][1ex]{$#1$}}}
\DeclareMathOperator{\prim}{pr}

\newcommand{\Rn}{\ensuremath{\R^{n}}}

\newcommand{\plust}{\text{+1}}
\newcommand{\minust}{\text{--1}}
\newcommand{\qmt}{\mathbfit{\top}}
\newcommand{\ovo}{\text{ovo}}
\newcommand{\ovr}{\text{ovr}}

\DeclareMathOperator{\Rob}{Rob}
\DeclareMathOperator{\rad}{rad}

\newcommand{\udr}{\stackrel{{\mbox{\tiny\ensuremath{\triangle}}}}{\Leftrightarrow}}

\newcommand{\ra}{\rightarrow}
\newcommand{\Ra}{\Rightarrow}
\newcommand{\Lra}{\Leftrightarrow}
\newcommand{\ud}{\triangleq}

\def \tuple#1{\langle #1 \rangle}

\newcommand{\bN}{\mathbb{N}}

\newcommand{\norm}[1]{\left\lVert #1 \right\rVert}

\newcommand{\Int}{\ensuremath{\mathrm{Int}}}

\newcommand{\commentt}[1]{}

\DeclareMathOperator{\sign}{sign}

\DeclareMathOperator{\AF}{AF}
\DeclareMathOperator{\RAF}{RAF}
\DeclareMathOperator{\frm}{frm}
\newcommand{\gammafk}{\gamma_{\AF_k}}

\renewcommand{\vec}[1]{\ensuremath{\mathbf{#1}}}

\title{
  Robustness Verification of Support Vector Machines
}

\author{Francesco Ranzato \and Marco Zanella}
\institute{Dipartimento di Matematica, University of Padova, Italy}

\date{}

\begin{document}
\maketitle

\begin{abstract}
We study the problem of formally verifying 
the robustness to adversarial examples  
of support vector machines (SVMs), a major machine learning model for classification and regression tasks. 
Following a recent stream of works on formal robustness verification of (deep) neural networks, 
our approach relies on a sound abstract version  
of a given SVM classifier to be used for checking its robustness. This methodology is parametric on a given 
numerical abstraction of real values and, analogously to the case of neural networks, needs neither 
abstract least upper bounds nor
widening operators on this abstraction. 
The standard interval domain provides a simple instantiation of our abstraction technique, 
which is enhanced with the domain of reduced affine forms, which is an efficient abstraction of the zonotope abstract domain. 
This robustness verification technique has been fully implemented and experimentally evaluated on SVMs based on linear
and nonlinear (polynomial and radial basis function) kernels,  
which have been trained on the popular 
MNIST dataset of images and on the recent and more challenging Fashion-MNIST dataset. 
The experimental results of our prototype SVM robustness verifier appear to be 
encouraging: this automated verification is 
fast, scalable and shows significantly high percentages of provable robustness on the test set
of MNIST, in particular compared to the analogous provable robustness of neural networks.      
\end{abstract}

\section{Introduction}
\label{sec:introduction}

Adversarial machine learning \cite{cacm18,aml-scale,aml-book} is an emerging hot topic studying vulnerabilities of machine learning (ML)
techniques
in adversarial scenarios and whose main objective is to design methodologies for making learning tools robust to adversarial 
attacks. Adversarial examples have been found in diverse application fields of ML such as
image classification, speech recognition and malware detection \cite{cacm18}. 
Current defense techniques include adversarial model training,  input validation, testing and automatic verification 
of learning algorithms (see the recent survey \cite{cacm18}). In particular, formal verification of ML classifiers started to be
an active field of investigation \cite{dillig2019,ehlers17,vechev-sp18,deepsafe,huang17,katz17,mirman2018differentiable,pt10,pt12,vechev-nips18,singh2019,wang18,weng18}
within the verification and static analysis community. 
Robustness to adversarial inputs is an important safety property of
ML classifiers whose formal verification has been investigated for (deep) neural networks \cite{dillig2019,vechev-sp18,pt10,vechev-nips18,singh2019,weng18}. 
A classifier is robust to some (typically small) perturbation of its input objects representing an adversarial attack 
when it assigns the same class to all the objects within that perturbation. Thus, 
slight malicious alterations of input objects should not deceive a robust classifier. Pulina and Tacchella \cite{pt10} first put forward
the idea of a formal robustness verification of neural network classifiers by leveraging 
interval-based abstract interpretation for designing a sound abstract classifier. This abstraction-based 
verification approach has been pushed forward by Vechev et al.\ \cite{vechev-sp18,vechev-nips18,singh2019},
who designed a scalable robustness verification technique which relies on abstract interpretation of deep neural networks based
on a specifically tailored abstract domain \cite{singh2019}. 

While all the aforementioned verification techniques consider (deep) neural networks as ML model, in this work we focus
on support vector machines (SVMs), which is a major learning model extensively and successfully used  
for both classification and regression tasks \cite{cs00}. SVMs are widely applied in different fields where adversarial attacks 
must be taken into account, notably image classification, malware detection, intrusion detection and spam filtering \cite{biggio2014}. 
Adversarial attacks and robustness issues of SVMs have been defined and studied
by some authors \cite{biggio2014,biggio11,Nam2010,trafalis2007robust,biggio15,xu2009robustness,Zhou2012},
in particular investigating robust training and experimental robustness evaluation of SVMs. To the best of our knowledge, no formal
and automatic robustness certification technique for SVMs has been studied.

\paragraph{\textbf{Contributions.}}
A simple and standard model of
adversarial region for a ML classifier $C:X \ra L$, where $X\subseteq \R^n$ is the input space and $L$ is the set of classes, is based on 
a set of perturbations $P(\vec{x})\subseteq X$ of an input $\vec{x}\in X$ for $C$, which typically exploits some metric on $\R^n$ to quantify a similarity to $\vec{x}$. A classifier $C$ is robust on an input $\vec{x}$ 
for a perturbation $P$ when for all $\vec{x}'\in P(\vec{x})$, $C(\vec{x}')=C(\vec{x})$ holds, meaning that the adversary cannot attack the classification of $\vec{x}$ made by $C$ 
by selecting input objects from $P(\vec{x})$ \cite{carlini}. 
We consider the most effective SVM classifiers based on common 
linear and nonlinear kernels, in particular polynomial and Gaussian radial basis function (RBFs) \cite{cs00}. Our technique for formally verifying 
the robustness of $C$ is quite standard: 
by leveraging a numerical abstraction $A$ of sets of real vectors in $\wp(\R^n)$, we define a sound
abstract classifier $C^\sharp:A\ra L\cup \{\top\}$ and a sound abstract perturbation $P^\sharp:X \ra A$, in such a way that if $C^\sharp(P^\sharp(\vec{x}))=C(\vec{x})$ holds then $C$ is proved to be robust on $\vec{x}$ for the adversarial region $P$. As usual in static analysis, scalability and precision are the main issues in SVM verification. A robustness verifier has to scale with the number of support vectors of the SVM classifier $C$, which in turn 
depends on the size of the training dataset for $C$, which may be huge (easily 
tens/hundreds of thousands of samples). Moreover, 
the precision of a verifier may crucially depend 
on the relational information between the components, called features in ML, 
of input vectors in $\R^n$, whose number may be quite large (easily hundreds/thousands of features). For our robustness verifier, we 
used an abstraction which is a product of the standard nonrelational interval 
domain \cite{CC77} and of the so-called reduced affine form (RAF) abstraction, a relational domain representing the dependencies from the
components of input vectors. A RAF for vectors in $\R^n$ is given 
by $a_0 +\sum_{i=1}^n a_i\epsilon_i + a_r \epsilon_r$, where  $\epsilon_i$'s 
are symbolic variables ranging in [-1,1] and representing a dependence from
the $i$-th component of the vector, while $\epsilon_r$ is a further 
symbolic variable in [-1,1]
which accumulates all the approximations introduced by nonlinear operations such
as multiplication and exponential. 
RAFs can be viewed as 
a restriction to a given length (here the dimension $n$ of $\R^n$) of the zonotope domain used in static program analysis \cite{goubault2015}, which features
an optimal abstract multiplication \cite{skalna2017}, the crucial operation of
abstract nonlinear SVMs.  
We implemented our robustness verification method for SVMs in
a tool called \emph{SAVer} (\emph{S}vm \emph{A}bstract \emph{Ver}ifier), written in C. 
Our experimental evaluation of SAVer employed the popular 
MNIST \cite{mnist} image dataset and
the recent and more challenging alternative Fashion-MNIST dataset \cite{fmnist}. Both datasets contain grayscale images of $28 \!\times\! 28$ pixels,
represented by normalized vectors of floating-point numbers in $[0,1]^{784}$. 
Our benchmarks show the percentage of samples of the full 
test sets for which a SVM 
is proved to be robust (and, dually, vulnerable) for a given perturbation, 
the average verification times per sample, and the scalability 
of the robustness verifier w.r.t.\ the number of support vectors. 
We also compared SAVer to DeepPoly \cite{singh2019}, a 
robustness verification tool for  deep neural networks based on abstract interpretation.
Our experimental results indicate that SAVer is fast and scalable and that 
the percentage of robustness provable by SAVer for SVMs
is significantly higher than the robustness  provable by DeepPoly for deep neural networks.

\paragraph*{\textbf{Illustrative Example.}}
Figure~\ref{fig:example-intro} shows a toy binary SVM classifier for input vectors in $\mathbb{R}^2$,  
with four support vectors $\vec{sv_1}=(8,7)$, $\vec{sv_2}=(10,-4)$, $\vec{sv_3}=(8,1)$, $\vec{sv_4}=(9,-5)$ for 
a polynomial kernel of degree 2. The corresponding classifier $C:\R^2 \ra \{\minust,\plust\}$ is:
\begin{align*}
C(\vec{x}) & = \sign(\textstyle\sum_{i=1}^4 \alpha_i y_i (\vec{sv_i}\cdot \vec{x})^2 + b) \\
&=\sign(\alpha_1( 8x_1 \!+\! 7x_2)^2 \!\!-\! \alpha_2 (10x_1 \!-\! 4x_2)^2 \!\!-\! \alpha_3(8x_1 \!+\!  x_2)^2 
\!\!+\! \alpha_4 (9x_1 \!-\! 5x_2)^2 + b)
\end{align*}
where $y_i$ and $\alpha_i$ are, respectively, the classes ($\pm 1$) and weights of the support vectors $\vec{sv_i}$, with:
$\alpha_1 \approx 5.36 \times 10^{-4}$, $\alpha_2 \approx -3.78 \times 10^{-3}$, $\alpha_3 \approx -9.23 \times 10^{-4}$, 
$\alpha_4 \approx 4.17 \times 10^{-3}$,  $b \approx 3.33$.
The 
set of vectors $\vec{x}\in \R^2$ such that $C(\vec{x}) = 0$ defines 
the decision curve (in red) between  labels $\minust$ and $\plust$. 
We consider a point $\vec{p} = (5, 1)$ and an adversarial region 
$P_1(\vec{p})=\{\vec{x}\in \R^2 \mid \max(|x_1-p_1|, |x_2-p_2|) \leq 1\}$, which is the $L_\infty$  ball of radius $1$ centered in $\vec{p}$
and can be exactly represented by  the interval in $\R^2$ 
(i.e., box) $P_1(\vec{p})=
(x_1 \in [4,6], x_2\in [0,2])$. As shown by the figure, this classifier $C$ is robust on $\vec{p}$ for this perturbation
because for all $\vec{x}\in P_1(\vec{p})$, $C(\vec{x})=C(\vec{p})=\plust$.

\begin{figure}[t]
\centering
  \includegraphics[scale=0.09]{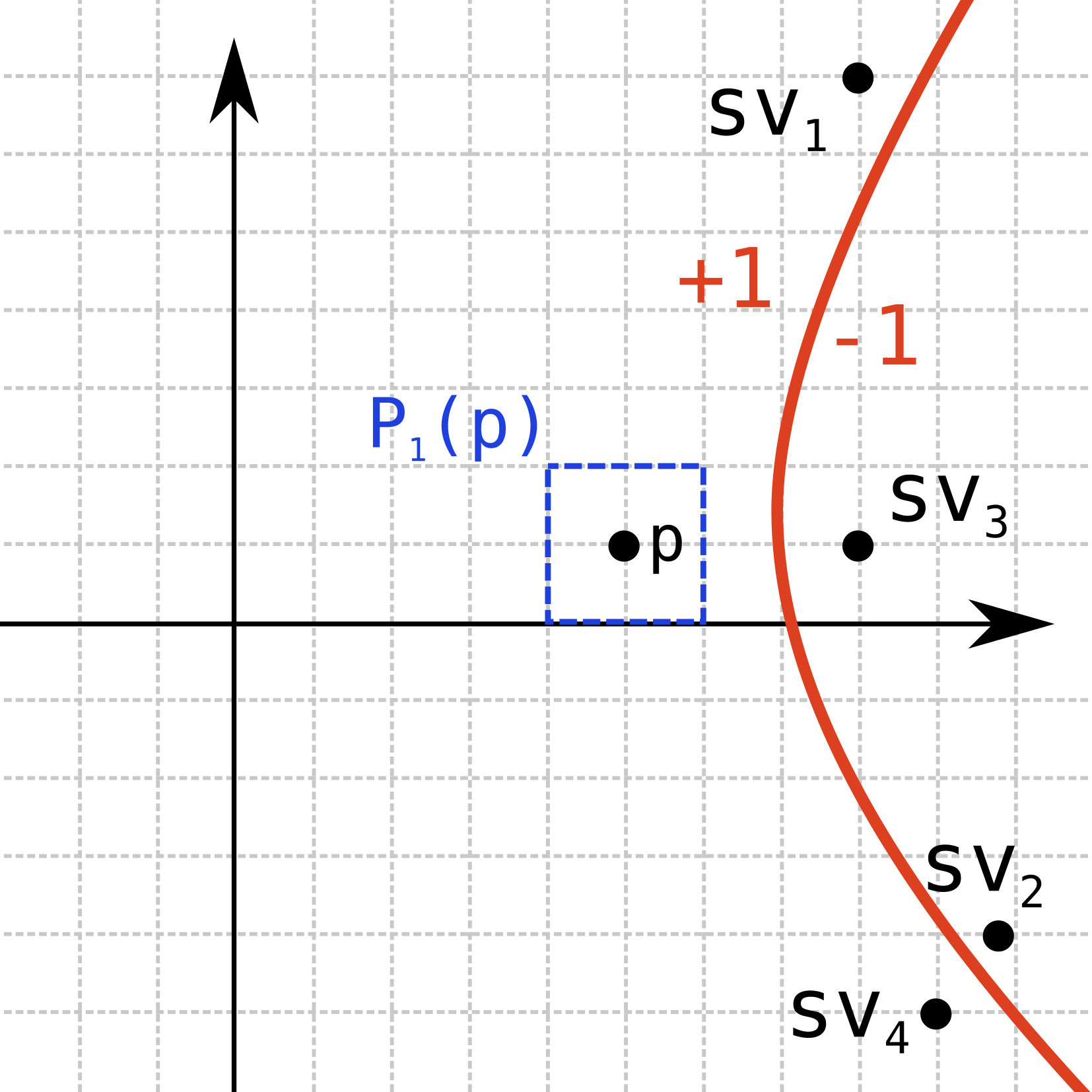}
  \caption{Example of SVM robustness.}
  \label{fig:example-intro}
\end{figure}%

\noindent
It turns out that the interval abstraction $C_\Int^\sharp$ of this classifier cannot prove the robustness of~$C$:
\begin{align*}
& C_\Int^\sharp(P_1(\vec{p})) = \sign(\textstyle \sum_{i=1}^4 \alpha_i y_i ((\vec{sv_i})_1[4,6] + (\vec{sv_i})_2[0,2])^2 +b)\\
& \quad = \sign(\alpha_1 y_1 [1024, 3844] + \alpha_2 y_2 [1024, 3600] +\alpha_3 y_3 [1024, 2500] + \alpha_4 y_4 [676, 2916] +b) \\
& \quad = \sign([-9.231596, 12.735958]) = \top
\end{align*}
Instead, the reduced affine form abstraction $C_{\RAF_2}^\sharp$ allows us to prove the robustness of $C$. Here, the perturbation 
$P_1(\vec{p})$ is exactly represented by the RAF 
$(\tilde{x}_1 = 5 + \epsilon_1, \tilde{x}_2 = 1 + \epsilon_2)$, where $\epsilon_1, \epsilon_2, \epsilon_r \in [-1, 1]$, 
and the abstract computation is as follows: 
\begin{align*}
& C_{\RAF_2}^\sharp (P_1(\vec{p}))  = \sign(\textstyle \sum_{i=1}^4 \alpha_i y_i [(\vec{sv_i})_1(5+\epsilon_1) + (\vec{sv_i})_2(1 +\epsilon_2)]^2 +b)\\
& \quad = \sign(
  \alpha_1 y_1 (47 + 8 \epsilon_1 + 7 \epsilon_2)^2 +
  \alpha_2 y_2 (46 + 10 \epsilon_1 - 4 \epsilon_2)^2 + \\ 
  & \quad \qquad\quad\;\,\, 
\alpha_3 y_3 (41 + 8 \epsilon_1 + \epsilon_2)^2 
 + \alpha_4 y_4 (40 + 9\epsilon_1 - 5 \epsilon_2)^2 
 + b) \\
& \quad = \sign(
    \alpha_1 y_1 (2322 + 752\epsilon_1 + 658\epsilon_2 + 112 \epsilon_r)
  + \alpha_2 y_2 (2232 + 920\epsilon_1 - 368\epsilon_2 + 80 \epsilon_r) + \\
& \quad \qquad\quad\;\,\, \alpha_3 y_3 (1746 + 656\epsilon_1 + 82\epsilon_2 + 16 \epsilon_r)
  + \alpha_4 y_4 (1706 + 720\epsilon_1 - 400\epsilon_2 + 90\epsilon_r)
  + b) \\
& \quad = \sign(1.635264 -0.680779\epsilon_1 +0.001047\epsilon_2 + 0.753025\epsilon_r) =\plust
\end{align*}
Here, the RAF analysis is able to prove that $C$ is robust on $P_1(\vec{p})$, since the final RAF has an interval range
$[0.200413, 3.070115]$ which is always positive. 

\section{Background}
\label{sec:svns-and-robustness}

\paragraph*{\textbf{Notation.}}
If $\vec{x},\vec{y} \in \Rn$, $z\in \R$ and $i\in [1,n]$ then $\vec{x}_i=\pi_i(\vec{x})\in \R$, 
$\vec{x}\cdot \vec{y} \ud \sum_i \vec{x}_i\vec{y}_i\in \R$, $\vec{x}+\vec{y}\in \Rn$, $z\vec{x}\in \Rn$, 
$\norm{\vec{x}}_2 \ud \sqrt{\vec{x}\cdot \vec{x}}\in \R$, $\norm{\vec{x}}_{\infty} \ud {\max\{ |\vec{x}_i| \mid i\in [1,n]\}} \in \R$,
denote, respectively, $i$-th component, dot product, vector addition, scalar multiplication, 
$L_2$ (i.e., Euclidean) and 
$L_\infty$ (i.e., maximum) norms  in $\Rn$. 
If $h:X \ra Y$ is any function then $h^c: \wp(X)\ra \wp(Y)$ defined by $h^c(S) \ud \{h(x)~|~x\in S\}$ 
denotes the standard collecting 
lifting of $h$, and, when clear from the context, we slightly abuse notation by using $h(S)$ instead of $h^c(S)$. 

\paragraph*{\textbf{Classifiers and Robustness.}} 
Consider a training dataset $T=\{(\vec{x_1},y_1),...,(\vec{x}_\mathbfit{N},y_N)\}$ $\subseteq$ $X\times L$, where $X\subseteq 
\R^n$ is the input space,
$\vec{x}_\mathbfit{i}\in X$ is called feature (or attribute) vector and $y_i$ is its label (or class) ranging into
the output space $L$.
A supervised learning algorithm  $\mathcal{SL}:\wp(X\times L) \ra (X \ra L)$ (also called trainer) computes a classifier function
$\mathcal{SL}(T):X \ra L$ ranging in some function subspace (also called hypothesis space). The learned classifier $\mathcal{SL}(T)$ is a function
that best fits the training dataset $T$ according to a principle of empirical risk minimization. 
The machine learning algorithm $\mathcal{SL}$
computes a classifier $\mathcal{SL}(T)$ by solving 
a complex optimization problem.  
The output space is assumed to be represented by real numbers, 
i.e., $L\subseteq \R$, and for binary classifiers with $|L|=2$, the standard assumption is that $L=\{\minust,\plust\}$. 

The standard threat model \cite{carlini,cacm18} of untargeted 
adversarial examples for a generic classifier $C:X \ra L$ is as follows. Given a valid input object $\vec{x}\in X$
whose correct
label is $C(\vec{x})$, an adversarial example for $\vec{x}$ is a legal input  $\vec{x}'\in X$ such that $\vec{x}'$ is a small perturbation of 
(i.e., 
is similar to) $\vec{x}$ and $C(\vec{x}')\neq C(\vec{x})$. An adversarial region is the set of perturbations $P(\vec{x})\subseteq X$ that 
the adversary is allowed to make to $\vec{x}$, meaning that a function $P:X\ra \wp(X)$ 
models an adversarial region.  
A perturbation $P(\vec{x})$ is typically modeled by some distance metric to quantify 
a similarity to $\vec{x}$, usually a $p$-norm, and  the most general model of perturbation 
simply requires that for all $\vec{x}\in X$, $\vec{x}\in P(\vec{x})$. 
A typical distance metric, as considered in
\cite{carlini,vechev-sp18,singh2019}, is determined by the $L_\infty$ norm: given (a small) 
$\delta >0$, $P_\delta^\infty(\vec{x})\ud \{\vec{x}'\in \R^n \mid
\norm{\vec{x}'-\vec{x}}_\infty \leq \delta\}=\{\vec{x}'\in \R^n \mid \forall i.\: \vec{x}'_i \in [\vec{x}_i-\delta,\vec{x}_i+\delta]\}$.
A classifier $C$ is defined to be robust on an input vector $\vec{x}$ 
for an adversarial region $P$ when for all $\vec{x}'\in P(\vec{x})$, $C(\vec{x}')=C(\vec{x})$ holds, denoted by
$\Rob(C,\vec{x},P) \udr \{C(\vec{x}') \mid \vec{x}'\in P(\vec{x})\}=\{C(\vec{x})\}$. This means that
the adversary cannot attack the classification of $\vec{x}$ made by $C$ 
by selecting input objects from the region $P(\vec{x})$. 

\paragraph*{\textbf{Support Vector Machines.}}
Several strategies and optimization techniques are available to train a SVM, but they are not relevant for our purposes 
(\cite{cs00} is a popular standard reference for SVMs).  
A SVM classifier partitions the input space $X$ into regions, each representing a class of the output space $L$. 
In its simplest formulation, the learning algorithm produces a linear SVM binary classifier with $L=\{\minust,\plust\}$ which relies 
on a hyperplane of $\R^n$ that separates training vectors labeled by $\minust$ from vectors labeled $\plust$.  
The training phase consists in finding (i.e., learning) such hyperplane. While many separating hyperplanes may exist, 
the SVM separating hyperplane has the maximum distance (called margin) with the closest vectors in the training dataset, 
because a maximum-margin learning algorithm statistically reduces the generalization error. This SVM hyperplane
is univocally represented by its normal vector  $\vec{w} \in \mathbb{R}^n$ and by a displacement scalar $b \in \mathbb{R}$, 
so that the hyperplane equation is $\vec{w} \cdot \vec{x} = b$. 
The classification of an input vector $\vec{x}\in X$ therefore boils down to determining the half-space containing $\vec{x}$, namely,
the linear binary classifier is the decision function $C(\vec{x}) = \sign(\vec{w} \cdot \vec{x} - b)$, where the case $\sign(0)=0$ is negligible 
(e.g.\ $\sign(0)$ may assign the class $\plust$). This linear classifier $\sign(\vec{w} \cdot \vec{x} - b)$ is 
in so-called primal form, while nonlinear classifiers are instead in dual form and based on a so-called kernel function. 

When the training set $T$ cannot be linearly separated in a satisfactory way, $T$ is projected into a much higher dimensional space 
through a projection map $\varphi: {\mathbb{R}^n \rightarrow \mathbb{R}^k}$, with $k > n$,
where $\varphi(T)$
may become linearly separable.  
Training a SVM classifier boils down to a  high-dimensional quadratic programming problem which can be solved either in its primal or dual form. 
When solving the dual problem, the projection function $\varphi$ is only involved in dot products 
$\varphi(\vec{x}) \cdot \varphi(\vec{y})$ in $\R^k$, 
so that this projection is not actually needed if these dot products in $\R^k$ can be equivalently formulated through a function 
$k: \mathbb{R}^n \times \mathbb{R}^n \rightarrow \mathbb{R}$, called kernel function, 
such that $k(\vec{x}, \vec{y}) = \varphi(\vec{x}) \cdot \varphi(\vec{y})$. 
Given a dataset $T=\{(\vec{x_1},y_1),...,(\vec{x}_\mathbfit{N},y_N)\}$, with $y_i\in \{\minust,\plust\}$, 
solving the dual problem for training the SVM classifier means finding a set
$\{\alpha_i\}_{i=1}^N \subseteq \mathbb{R}$, called set of weights, 
which maximizes the 
following function $f:\R^N \ra \R$:
\[
\textstyle
 \max f(\alpha_1,...,\alpha_N) \ud  \sum_{i = 1}^N \alpha_i - \frac{1}{2} \sum_{i, j = 1}^N \alpha_i \alpha_j y_i y_j k(\vec{x_i}, \vec{x_j})
\]
subject to: for all $i$, $0 \leq \alpha_i \leq c$, where $c\in \R_{>0}$ is a tuning parameter, 
and $\sum_{i = 1}^N \alpha_i y_i = 0$. 
This set of weights defines the following SVM binary classifier $C$: for all input $\vec{x}\in X \subseteq \mathbb{R}^n$, 
\begin{equation}\label{classifier}
C(\vec{x}) \ud \sign ( \textstyle [\textstyle\sum_{i = 1}^N \alpha_i y_i k(\vec{x_i}, \vec{x})] -b ) 
\end{equation}
for some offset parameter $b \in \mathbb{R}$. By defining $D_k(\vec{x}) \ud \sum_{i = 1}^N \alpha_i y_i k(\vec{x_i}, \vec{x})$, 
this classifier will be also denoted by $C(\vec{x}) = \sign(D_k(\vec{x})-b)$. 
In practice most weights $\alpha_i$ are $0$, hence only a subset of the training vectors $\vec{x_i}$ is actually used by the SVM classifier $C$, 
and these are called support vectors. 
By a slight abuse of notation, 
we will assume that $\alpha_i \neq 0$ for all $i\in [1,N]$, namely $\{\vec{x_i}\}_{i=1}^N \subseteq \R^n$ denotes the set of
support vectors extracted from the training set by the SVM learning algorithm for some kernel function. 
We will consider the most common and effective kernel functions used in SVM training: 
\begin{itemize}
\item[{\rm (i)\/}] linear kernel: $k(\vec{x},\vec{y})=\vec{x}\cdot \vec{y}$;  
\item[{\rm (ii)\/}] $d$-polynomial kernel: $k(\vec{x},\vec{y})= (\vec{x}\cdot \vec{y}+c)^d$ (common powers are $d=2,3,9$); \item[{\rm (iii)\/}] Gaussian radial basis function (RBF):   $k(\vec{x},\vec{y})= e^{-\gamma \norm{\vec{x}-\vec{y}}^2_2}$, for some $\gamma>0$.
\end{itemize}

\paragraph*{\textbf{SVM Multiclass Classification.}}
Multiclass datasets have a finite set of labels $L=\{y_1,...,y_m\}$ with $m>2$. The standard approach to multiclass classification problems 
consists in a reduction into  multiple binary classification problems using one of the following two simple strategies \cite{multiclass}.
In the ``one-versus-rest'' (ovr) strategy, $m$ binary classifiers are trained, where each binary classifier 
$C_{i,\bar{i}}$ determines whether an input vector $\vec{x}$
belongs to the class $y_i \in L$ or not by assigning a real confidence score for its decision 
rather than just a label, so that the class $y_m$ with the highest-output confidence score is the class assigned to $\vec{x}$.
Multiclass SVMs using this ovr approach  might not work satisfactorily
because the ovr approach often leads to unbalanced datasets already for a few classes 
due to unbalanced partitions into $y_i$ and $L\smallsetminus \{y_i\}$. 
\\
\indent
The most common solution \cite{multiclass} is to follow a ``one-versus-one'' (ovo) approach, where $m (m - 1) / 2$ binary 
classifiers $C_{\{i,j\}}$ are trained on the restriction of the original training set to vectors with 
labels in $\{y_i,y_j\}$, with $i\neq j$, so that 
each $C_{\{i,j\}}$ determines whether an input vector 
belongs (more) to the class $y_i$ or (more to) $y_j$. Given an input vector $\vec{x}\in X$ each of these $m (m - 1) / 2$ binary classifiers $C_{\{i,j\}}(\vec{x})$ 
assigns a ``vote'' to one class in $\{y_i,y_j\}$, and at the end the class with the most votes wins, i.e., the argmax of the function
$\votes(\vec{x},y_i) \ud |\{j\in \{1,...,m\} \mid j\neq i,\; C_{\{i,j\}}(\vec{x})=y_i\}|$ is the winning class of $\vec{x}$.
Draw is a downside of the ovo strategy because it may well be the case
that for some (regions of) input vectors
multiple classes collect the same number of votes and therefore no classification can be done. In case of draw, a common
strategy \cite{multiclass} is to output any of the winning classes (e.g., the one with the smaller index). However, 
since our primary 
focus will be on soundness
of abstract classifiers, we need to model an ovo multiclass classifier by a function $M_{\ovo}:X\ra \wp(L)$ defined by 
\[
M_{\ovo}(\vec{x}) \ud \{y_k \in L \mid k\in \argmax_{i\in \{1,...,m\}} \votes(\vec{x},y_i)\},
\]
so that $|M_{\ovo}(\vec{x})|>1$ models a draw in the ovo voting. 
Our experiments used SVM multi-classifiers which have been trained 
with the ovo voting procedure, which is the standard multi-classification 
scheme adopted by the popular
Scikit-learn framework \cite{scikit}.

\paragraph*{\textbf{Numerical Abstractions.}}
According to the most general definition, a numerical abstract domain is a tuple ${\tuple{A,\leq_A,\gamma}}$ where 
$\tuple{A,\leq_A}$ is at least a preordered set and the concretization function $\gamma:A\ra \wp(\Rn)$, with $n\geq 1$, 
preserves the relation $\leq_A$, 
i.e., $a\leq_A a'$ implies $\gamma(a)\subseteq \gamma(a')$ (i.e., $\gamma$ is monotone). Thus, $A$ plays the usual role of set of 
symbolic representations for sets of vectors of $\Rn$. Well-known examples of numerical abstract domains include
intervals, zonotopes, octagons, convex polyhedra
(we refer to the tutorial \cite{min17}). Some numerical domains just form preorders
(e.g.,  standard representations of 
octagons by DBMs allow multiple representations) while other domains give rise to posets
(e.g., intervals). 
Of course, any preordered abstract domain $\tuple{A,\leq_A,\gamma}$ can be quotiented to a poset $\tuple{A_{/\cong},\leq_A,\gamma}$  where $a\cong a'$ iff $a\leq_A a'$ and $a'\leq_A a$. 
While a monotone 
concretization $\gamma$ is enough for reasoning about soundness of static analyses on numerical domains, 
the notion of best correct approximation of concrete sets relies on the existence of 
an abstraction function $\alpha:\wp(\Rn)\ra A$ which requires that $\tuple{A,\leq_A}$ is (at least) a poset and 
that the pair $(\alpha, \gamma)$ forms a Galois connection, 
i.e.\ for any $X \subseteq \Rn, a \in A$, $\alpha(X) \leq_A a \Lra X\subseteq \gamma(a)$ holds, which becomes a Galois insertion when $\gamma$ is injective 
(or, equivalently, $\alpha$ is surjective). 
Several numerical domains admit a definition through
Galois connections, in particular intervals and octagons, while 
some domains do not have an abstraction map, notably zonotopes and convex polyhedra. 

Consider a concrete $k$-ary real operation $f:\wp(\Rn)^k\ra \wp(\Rn)$, for some $k\in \bN_{>0}$, and a corresponding
abstract map $f^\sharp:A^k\ra A$. Then, $f^\sharp$ is a correct (or sound) approximation of $f$ when
$f \circ \vv{\gamma} \mathrel{\subseteq} \gamma \circ f^\sharp$ holds,
where $\vv{\gamma}:A^k \ra \wp(\Rn)^k$ denotes
the $k$-th product of $\gamma$. 
Also, $f^\sharp$ is exact (or forward-complete) 
when $f\circ \vv{\gamma} = \gamma\circ f^\sharp$ holds. When 
a Galois connection $(\alpha,\gamma)$ for $A$ exists, if $f^\sharp$ is exact then it coincides with 
the best correct approximation (bca) of $f$ on $A$, which is the abstract 
function $\alpha \circ f \circ \vv{\gamma}:A^k \ra A$.

The abstract domain $\Int$ of numerical intervals on the poset of real numbers $\tuple{\R\cup \{\text{--}\infty,\text{+}\infty\},\leq}$
 is defined as usual \cite{CC77}:
\[
\Int \ud \{\bot,[\text{--}\infty,\text{+}\infty]\} \cup 
\{[l,u] ~|~ l,u\in \R, l\leq u\} \cup \{[\text{--}\infty,u] ~|~ u\in \R\} 
\cup \{[l,\text{+}\infty] ~|~ l \in \R \}.
\]
The concretization map $\gamma:\Int \ra \wp(\R)$ is standard:
\begin{align*}
&\gamma(\bot)\ud \varnothing,\quad \gamma(\top)\ud \R,\quad\gamma([l,u])\ud \{x\in \R \mid l \leq x\leq u\},\\
&\gamma([\text{--}\infty,u])\ud \{x\in \R \mid x\leq u\},\quad \gamma([l,\text{+}\infty])\ud \{x\in \R \mid l\leq x\}.
\end{align*}
Intervals admit an abstraction map $\alpha:\wp(\R) \ra \Int$ such that $\alpha(X)$ is the 
least interval containing $X$, 
i.e., $\alpha(X) \ud [\inf X, \sup X]$ if $X\neq \varnothing$ with $\inf X, \sup X\in \R\cup\{\text{--}\infty,\text{+}\infty\}$, 
while $\alpha(\varnothing)=\bot$.
Thus, $(\alpha,\gamma)$ define a Galois insertion between $\tuple{\Int,\sqsubseteq}$ and $\tuple{\wp(\R),\subseteq}$ w.r.t.\ the standard interval partial order $I\sqsubseteq I' \Lra \gamma(I)\subseteq \gamma(I')$.

\section{Abstract Robustness Verification Framework}\label{sec-arvf}
Let us describe a sound abstract robustness verification framework for binary and multiclass SVM classifiers.   
We first consider a binary classifier $C:X \ra \{\minust,\plust\}$, where $C(\vec{x}) = \sign(D(\vec{x}) - b)$ and $D:X\ra \R$ has been trained 
for some kernel function $k$. We also consider a given adversarial region $P: X \ra \wp(X)$ for $C$. 
Consider a numerical abstract domain $\tuple{A,\leq_A}$ whose abstract values represent sets of input vectors
for a binary classifier $C$, namely $\gamma:A\ra \wp(X)$, where $X$ is the input space of $C$. We use $A_n$ to emphasize that $A$ is used as an 
abstraction of properties of $n$-dimensional vectors in $\R^n$, so that $A_1$ denotes that $A$ is used as an abstraction of 
sets of scalars in $\wp(\R)$. 
 
\begin{definition}[Sound Abstract Classifiers and Robustness Verifiers]\rm
A \emph{sound abstract binary classifier} on $A$ is an algorithm $C^\sharp:A\ra 
\{ \minust, \plust, \qmt\}$ such that for all $a\in A$,
$C^\sharp(a) \neq \qmt \Ra \forall \vec{x}\in \gamma(a).\: C(\vec{x})=C^\sharp(a)$. 

\noindent
An abstract perturbation is a computable function $P^\sharp:X\ra A$. 
The pair $\tuple{C^\sharp,P^\sharp}$ is a \emph{sound robustness verifier} for $C$ w.r.t.\ $P$ when
for all $\vec{x}\in X$, if $C^\sharp(P^\sharp(\vec{x}))=C(\vec{x})$ then $\Rob(C,\vec{x}, P)$ holds.  
\qed
\end{definition}

Hence, a sound abstract classifier $C^\sharp$ is allowed to output a ``don't know'' answer $\qmt$ on abstract inputs, but when it provides 
a classification this must be correct. Also, a sound abstract verifier $\tuple{C^\sharp,P^\sharp}$ 
proves the robustness of an input $\vec{x}$ 
when the classification $C(\vec{x})$ is preserved by $C^\sharp$ on the abstract perturbation $P^\sharp(\vec{x})$. 

The key step for designing an abstract robustness verifier 
is to design a sound abstract version of the trained function $D:X\ra \R$ on the abstraction $A$, namely
an algorithm $D^\sharp: A_n \ra A_1$ such that, for all $a\in A_n$, $D^c(\gamma(a)) \subseteq \gamma(D^\sharp(a))$. 
We also need that the abstraction $A$ is endowed with a
sound approximation
of the Boolean test $\sign_b(\cdot): \R \ra \{ \minust, \plust\}$ for any bias $b\in \R$, where 
$\sign_b(x) \ud \textbf{if}~x\geq b~\textbf{then}~{\plust}~\textbf{else}~{\minust}$. 
Hence, we require a computable abstract
function $\sign_b^\sharp : A_1 \ra \{ \minust, \plust, \qmt\}$  which is sound for $\sign_b$, that is,
for all $a\in A_1$,
$\sign_b^\sharp(a)\neq \qmt \Ra \forall x\in \gamma(a). \sign_b(x)=\sign_b^\sharp(a)$.  
These hypotheses therefore provide a straightforward sound abstract classifier $C^\sharp: A \ra \{\minust,\plust,\qmt\}$ defined 
as follows: 
\[
C^\sharp(a) \ud \sign_b^\sharp(D^\sharp(a)).
\]
Also, the  abstract perturbation
$P^\sharp:X \ra A_n$ has to be a sound approximation of $P$, meaning that
for all $\vec{x}\in X$, $P(\vec{x})\subseteq \gamma(P^\sharp(\vec{x}))$. It turns out that 
these hypotheses entail the soundness of the robustness verifier.

\begin{lemma}\label{srv-lemma}
$C^\sharp$ is a sound abstract classifier and
$\tuple{C^\sharp,P^\sharp}$ is a sound robustness verifier.
\end{lemma} 
\begin{proof}
Let $a\in A$ and assume that 
$C^\sharp(a) = \sign_b^\sharp(D^\sharp(a))  \neq \top$. Then, 
by soundness of $\ok{\sign^\sharp_b}$, we have that for all $y\in \gamma( D^\sharp(a))$, 
$\sign_b(y) = C^\sharp(a)$, and in turn, by soundness of $D^\sharp$, 
for all $\vec{x}\in \gamma(a)$,
$\sign_b(D(\vec{x})) = C^\sharp(a)$.   \\
Let us show that $\tuple{C^\sharp,P^\sharp}$ is a sound verififer. Let $\vec{x}\in X$ and  assume that $C^\sharp(P^\sharp(\vec{x}))=C(\vec{x})$. 
Since $C^\sharp(P^\sharp(\vec{x}))\neq \qmt$, by soundness of $C^\sharp$ in point~(i), we have that for all
$\vec{x}' \in \gamma(P^\sharp(\vec{x}))$, $C(\vec{x}')=C^\sharp(P^\sharp(\vec{x}))=C(\vec{x})$. Thus, by soundness of $P^\sharp$, for all $\vec{x}'\in P(\vec{x})$,
$C(\vec{x}')=C(\vec{x})$, i.e., $\Rob(C,\vec{x}, P)$ holds.  \qed
\end{proof}

Thus, if $T$ is a test subset of the dataset used for training the classifier 
$C$ then we may correctly assert that  $C$ is provably $q\%$-robust on $T$ for the perturbation $P$
when the abstract robustness verification of Lemma~\ref{srv-lemma}~(ii) is able to check that $C$ is robust on 
$q\%$ of the test vectors in $T$. Of course, by soundness, this means that $C$ is certainly robust on \emph{at least} 
$q\%$ of  the inputs in $T$, while on the remaining $(100-q)\%$ of $T$ we do not know: these could be spurious or real
unrobust input vectors. 

In order to 
design a sound abstract version of  $D(\vec{x}) = \sum_{i = 1}^N \alpha_i y_i k(\vec{x_i}, \vec{x})$ 
we surely need sound approximations on $A_1$
of scalar multiplication and addition. Thus, we require a sound abstract scalar multiplication 
$\lambda a.za: A_1 \ra A_1$, for any $z\in \R$,  such that for all $a\in A_1$, $z\gamma(a) 
\subseteq \gamma(za)$, and a sound addition $+^\sharp: A_1 \times A_1 \ra A_1$ such that for all $a,a'\in A_1$, 
$\gamma(a) + \gamma(a') \subseteq \gamma(a+^\sharp a')$, and we use $\sum_{i\in I}^\sharp a_i$ to denote 
an indexed abstract summation. 

\subsection{Linear Classifiers}\label{alc-sec}
Sound approximations of scalar multiplication and addition are enough for designing a sound robustness verifier for
a linear classifier. As a preprocessing step,  
for a binary classifier $C(\vec{x}) = \sign ([\sum_{i = 1}^N \alpha_i y_i (\vec{x_i} \cdot \vec{x})] -b)$
which has been trained for the linear kernel, 
we preliminarly compute the hyperplane normal vector $\vec{w}\in \R^n$: for all $j\in[1,n]$, 
$\vec{w}_j \ud \sum_{i = 1}^N \alpha_i y_i \vec{x_i}_j$, so that for all $\vec{x}\in \R^n$, 
$\vec{w}\cdot \vec{x} = \sum_{j=1}^n \vec{w}_j \vec{x}_j = \sum_{i = 1}^N \alpha_i y_i (\vec{x_i} \cdot \vec{x})$. 
Thus, $C(\vec{x})=\sign([\sum_{j=1}^n \vec{w}_j \vec{x}_j]-b)$
is the linear classifier in primal form, whose robustness can be abstractly verified by resorting to just sound abstract scalar multiplication and addition
on $A_1$. The  noteworthy advantage of abstracting a classifier in primal form is that each component of 
the input vector $\vec{x}$ occurs just once in 
$\sign([\sum_{j=1}^n \vec{w}_j \vec{x}_j]-b)$, while in the dual form $\sign ([\sum_{i = 1}^N \alpha_i y_i (\vec{x_i} \cdot \vec{x})] -b)$
each component $\vec{x}_j$ occurs exactly $N$ times (one for each support vector), so that a precise abstraction of
this latter dual form would be able to represent 
the correlation between (the many) multiple occurrences of each $\vec{x}_j$. 

\subsection{Nonlinear Classifiers}
Let us consider a nonlinear kernel binary classifier 
$C(\vec{x}) = \sign(D(\vec{x}) - b)$, where $D(\vec{x})= \sum_{i = 1}^N \alpha_i y_i k(\vec{x_i}, \vec{x})$ 
and $\{\vec{x_i}\}_{i=1}^N \subseteq \R^n$ is the set of
support vectors for the  kernel function $k$. 
Thus,  what we additionally need  here is a sound abstract kernel function $k^\sharp:\R^n \times A_n \ra A_1$ such that for any
support vector $\vec{x_i}$ and $a\in A_n$, 
$\{k(\vec{x_i},\vec{x}) \mid \vec{x}\in \gamma(a)\} \subseteq \gamma(k^\sharp(\vec{x_i},a))$. 
Let us consider the polynomial and 
RBF kernels.

For a $d$-polynomial kernel $k(\vec{x},\vec{y}) = (\vec{x}\cdot \vec{y} + c)^d$, we need
sound approximations of the unary dot product $\lambda \vec{y}.\, \vec{x}\cdot \vec{y}:\R^n \ra \R$, for any given
$\vec{x}\in \R^n$, and 
of the $d$-power function $(\cdot)^d: \R \ra \R$. Of course, a sound nonrelational 
approximation of  $\lambda \vec{y}.\,\vec{x}\cdot \vec{y}=
\sum_{j=1}^n \vec{x}_j \vec{y}_j$
can be obtained simply by using sound abstract scalar multiplication and addition on $A_1$.  
Moreover, a sound abstract binary multiplication provides 
a straightforward definition of 
a sound abstract $d$-power function $(\cdot)^{d^\sharp}: A_1 \ra A_1$. 
If   $*^\sharp: A_1 \times A_1 \ra A_1$ is a sound abstract multiplication
such that for all $a,a'\in A_1$, 
$\gamma(a)*\gamma(a') \subseteq \gamma(a*^\sharp a')$, then a sound 
abstract $d$-power procedure can be defined simply by iterating the abstract multiplication $*^\sharp$
as follows:  for any $a\in A_1$, $d\geq 2$,
\begin{align*}
a^{d^\sharp} \ud \{r^\sharp := a*^\sharp a;~ \textbf{for}~ j:=3 ~\textbf{to}~d ~\textbf{do}~ r^\sharp := r^\sharp *^\sharp a;~ \textbf{return}~ r^\sharp\}
\end{align*}

For the RBF kernel $k(\vec{x},\vec{y})= e^{-\gamma \norm{\vec{x}-\vec{y}}^2_2} =  e^{-\gamma (\vec{x}-\vec{y})\cdot(\vec{x}-\vec{y})}$, 
for some $\gamma>0$, we need sound approximations of the self-dot product $\lambda \vec{x}. \vec{x}\cdot \vec{x}:\R^n \ra \R$, 
which is the squared Euclidean distance,   
and of the exponential $e^x: \R \ra \R$. 
Let us observe that 
sound abstract addition and multiplication induce a sound nonrelational approximation of the self-dot product:
 for all $\tuple{a_1,...,a_n}\in A_n$, 
$\tuple{a_1,...,a_n} \cdot^\sharp \tuple{a_1,...,a_n} \ud \sum^{\sharp n}_{j=1} a_j *^\sharp a_j$. 
Finally, we require 
a sound abstract exponential $e^{\sharp (\cdot)}: A_1 \ra A_1$
such that for all $a\in A_1$, $\{e^x \mid x\in \gamma(a)\} \subseteq \gamma(e^{\sharp a})$. 

\subsection{Abstract Multi-Classification}  \label{sec-absMC}
Let us consider multiclass classification for a set of labels $L=\{y_1,...,y_m\}$, with $m>2$.
It turns out that the multi-classification approaches based on a reduction to multiple binary classifications 
such as ovr and ovo introduce a further approximation in the abstraction process, because these
reduction strategies need to be soundly approximated. 
Let $M:X\ra \wp(L)$ be a multi-classifier, 
as modeled in Section~\ref{sec:svns-and-robustness} and an abstraction $A$ of $\wp(X)$. 
\begin{definition}[Sound Abstract Multi-Classifiers]\rm
A \emph{sound abstract multi-classifier} on $A$ is an algorithm $M^\sharp:A\ra \wp(L)$ such that
for all $a\in A$ and $\vec{x}\in \gamma(a)$, 
$M(\vec{x}) \subseteq M^\sharp(a)$. 
\qed
\end{definition}
Thus, an abstract multi-classifier $M^\sharp$ on some $a$ always provides an over-approximation to the classes
decided by $M$ on some input $\vec{x}$ represented by $a$, so that an output $L$ plays the role of a 
``don't know'' answer 
for $M^\sharp$.  

Let us first consider the ovr strategy and, for all $j\in [1,m]$, 
let $C_{j,\bar{j}}: X \ra \R$ denote the corresponding binary scoring classifier of 
$y_j$-versus-rest where 
$C_{j,\bar{j}}(\vec{x}) \ud D_j(\vec{x}) - b_j$. 
In order to have a sound approximation of ovr multi-classification, besides having $m$ sound abstract classifiers
$C_{j,\bar{j}}^\sharp: A_n \ra A_1$ such that for all $a\in A_n$, $\{C_{j,\bar{j}}(\vec{x})\in \R \mid \vec{x}\in
\gamma(a)\} \subseteq  \ok{\gamma(C^\sharp_{j,\bar{j}} (a))}$, it is needed
a sound abstract maximum function $\max^\sharp : (A_1)^m \ra \{1,...,m,\qmt\}$ such that 
if $(a_1,...,a_m)\in (A_1)^m$ and 
$(z_1,...,z_m)\in \gamma(a_1)\times ... \times \gamma(a_m)$ then 
$\max^\sharp(a_1,...,a_m) \neq \qmt \; \Ra \; 
\max(z_1,...,z_m) \in \gamma(a_{\max^\sharp(a_1,...,a_m)})$ holds. Clearly, as soon as the abstract function $\max^\sharp$ outputs $\qmt$, 
this abstract multi-classification scheme is inconclusive.
\begin{examplebf}\rm
Let $m=3$ and assume 
that an ovr multi-classifier $M_{\ovr}$ is robust on $\vec{x}$ for some adversarial region 
$P$ as a consequence of the following ranges of scores:  
for all $\vec{x}'\in P(\vec{x})$, 
$-0.5 \leq C_{1,\bar{1}}(\vec{x}') \leq -0.2$, $3.5\leq C_{2,\bar{2}}(\vec{x}') \leq 4$ and
$2 \leq C_{3,\bar{3}}(\vec{x}') \leq 3.2$. In fact, since the least score of $C_{2,\bar{2}}$ on the region $P(\vec{x})$ is  
greater than the greatest scores of $C_{1,\bar{1}}$ and $C_{3,\bar{3}}$ on $P(\vec{x})$, these ranges imply 
that for all $\vec{x}'\in P(\vec{x})$, $M(\vec{x}')=y_2$.
However, even in this advantageous scenario, on the abstract side we could not be able to 
infer that $C_{2,\bar{2}}$ always prevails over $C_{1,\bar{1}}$ and $C_{3,\bar{3}}$. For example, for 
the interval abstraction,
some interval binary classifiers for a sound perturbation $P^\sharp(\vec{x})$
could output the following sound intervals:
$\ok{C^\sharp_{1,\bar{1}}(P^\sharp(\vec{x}))} = [-1,-0.1]$, $\ok{C^\sharp_{2,\bar{2}}(P^\sharp(\vec{x}))} = [3.4,4.2]$ and
$\ok{C^\sharp_{3,\bar{3}}(P^\sharp(\vec{x}))} = [1.5,3.5]$. In this case, despite that each 
abstract binary classifier $\ok{C^\sharp_{i,\bar{i}}}$ is able to prove that $C_{i,\bar{i}}$ is robust on $\vec{x}$ for $P$ (because
the output intervals do not include 0), 
the ovr strategy here does not allow to conclude that the multi-classifier $M_{\ovr}$ is
robust on $\vec{x}$, because the lower bound of the interval approximation provided
by $C^\sharp_{2,\bar{2}}$ is not above the interval upper bounds of $\ok{C^\sharp_{1,\bar{1}}}$ and $\ok{C^\sharp_{3,\bar{3}}}$. In such a case,  
a sound abstract multi-classifier based on ovr cannot prove the robustness
of $M_{\ovr}$ for $P(\vec{x})$.
\qed
\end{examplebf}

Let us turn to the ovo approach which relies on $m (m - 1) / 2$ binary 
classifiers $C_{\{i,j\}}: X\ra \{i,j\}$. 
Let us assume that for all pairs $i\neq j$, a sound abstract binary classifier 
$\ok{C^\sharp_{\{i,j\}}}: A \ra \{y_i,y_j, \qmt\}$ is defined. Then, an abstract ovo multi-classifier 
$\ok{M_{\ovo}^\sharp}: A\ra \wp(L)$ can be defined as follows. 
For all $i\in \{1,...,m\}$ and $a\in A$, 
let $\votes^\sharp(a,y_i)\in \Int_{\N}$ be an interval of nonnegative integers used 
by the following abstract voting procedure \textsf{AV}, where $+^{\scriptscriptstyle\Int}$ denotes interval addition:
\begin{align}\label{def-av}
&\textbf{forall}~i\in [1,m]~\textbf{do}~\votes^\sharp(a,y_i):=[0,0];\nonumber\\
&\textbf{forall}~i,j\in [1,m] ~\textbf{s.t.}~i\neq j ~\textbf{do} \nonumber\\
&\;\;\: \textbf{if}~C^\sharp_{\{i,j\}}(a)=y_i~\textbf{then}~ \votes^\sharp(a,y_i) := \votes^\sharp(a,y_i) +^{\scriptscriptstyle\Int} [1,1];\\
&\;\;\: \textbf{elseif}~C^\sharp_{\{i,j\}}(a)=y_j~\textbf{then}~ \votes^\sharp(a,y_j) := \votes^\sharp(a,y_j) +^{\scriptscriptstyle\Int} [1,1];\nonumber\\
&\;\;\: \textbf{else}\votes^\sharp(a,y_i)\! := \votes^\sharp(a,y_i) +^{\scriptscriptstyle\Int} \! [0,1];
                      \votes^\sharp(a,y_j)\! := \votes^\sharp(a,y_j) +^{\scriptscriptstyle\Int} \! [0,1]; \nonumber
\end{align}
Let us notice that the last else branch is taken when $C^\sharp_{\{i,j\}}(a)=\qmt$, 
meaning that the abstract classifier $\ok{C^\sharp_{\{i,j\}}(a)}$ is not able to decide between
$y_i$ and $y_j$, so that in order to preserve the soundness of the abstract voting procedure, 
we need to increment just 
the upper bounds of the interval ranges of votes for both classes $y_i$ and $y_j$ while
their lower bounds do not have to change. Let us denote $\votes^\sharp(a,y_i) = [v^{\min}_i,v^{\max}_i]$. 
Hence, at the end of the \textsf{AV} procedure, 
$[v^{\min}_i,v^{\max}_i]$ provides an interval approximation of concrete votes as follows:
\begin{align*}
|\{j\neq i \mid \forall \vec{x}\in \gamma(a). C_{\{i,j\}}(\vec{x})=i\}| = v^{\min}_i
\quad\qquad 
|\{j\neq i \mid \exists \vec{x}\in \gamma(a). C_{\{i,j\}}(\vec{x})=i\}|  \leq v^{\max}_i
\end{align*}
The corresponding abstract multi-classifier is then defined as follows:
\begin{align}\label{eq-multiclass}
M_{\ovo}^\sharp(a) \ud \{ y_i \in L \mid \forall j\neq i.\, v^{\min}_j \leq v^{\max}_i\}
\end{align}
Hence, one may have an intuition for this definition by considering that a class $y_j$ is not in $\ok{M_{\ovo}^\sharp(a)}$ when there exists a different
class $y_k$ whose lower bound of votes is certainly strictly greater than 
the upper bound of votes for $y_i$.  
For example, for $m=4$, 
if $\votes^\sharp(a,y_1)=[4,4]$, $\votes^\sharp(a,y_2)=[0,2]$, $\votes^\sharp(a,y_3)=[4,5]$, $\votes^\sharp(a,y_4)=[1,3]$ then $M_{\ovo}^\sharp(a)=\{y_1,y_3\}$.  
\begin{examplebf}
\rm
Assume that $m=3$ and that
for all $\vec{x}'\in P(\vec{x})$, $M_{\ovo}(\vec{x}')=\{y_3\}$ because we have that
$\argmax_{i=1,2,3} \votes(\vec{x}',y_i) = \{3\}$. This means that a draw never happens for $M_{\ovo}$, 
so that for all $\vec{x}'\in P(\vec{x})$, $C_{\{1,3\}}(\vec{x}') = y_3$ and
$C_{\{2,3\}}(\vec{x}') = y_3$ certainly hold (because $m=3$). Let us also assume that 
$\{C_{\{1,2\}}(\vec{x}') \mid \vec{x}'\in P(\vec{x})\} =\{y_1,y_2\}$. Then, for a sound abstract perturbation 
$P^\sharp(\vec{x})$, we necessarily have that
$C^\sharp_{\{1,2\}}(P^\sharp(\vec{x}))=\top$. If we assume that
$C^\sharp_{\{1,3\}}(P^\sharp(\vec{x}))=y_3$ and
$C^\sharp_{\{2,3\}}(P^\sharp(\vec{x}))=\top $ hold then we have that
$M_{\ovo}^\sharp(P^\sharp(\vec{x})) = \{y_1,y_2,y_3\}$ because $\votes(P^\sharp(\vec{x}),y_1)=[0,1]$, $\votes(P^\sharp(\vec{x}),y_2)=[0,2]$
and $\votes(P^\sharp(\vec{x}),y_3)=[1,2]$. 
Therefore, in this case, $M_{\ovo}^\sharp$
is not able to prove the robustness of $M_{\ovo}$ on $\vec{x}$. 
Let us notice that 
the source of imprecision in this multi-classification is confined to 
the binary classifier $C^\sharp_{2,3}$ rather than the abstract voting $\textsf{AV}$ strategy. 
In fact, if we have that 
$C^\sharp_{\{1,3\}}(P^\sharp(\vec{x}))=\{y_3\}$ and
$C^\sharp_{\{2,3\}}(P^\sharp(\vec{x}))=\{y_3\}$ then 
$M_{\ovo}^\sharp(P^\sharp(\vec{x})) = \{y_3\}$, thus proving the robustness of $M$. 
\qed
\end{examplebf}

\begin{lemma}\label{lemma-ovo}
Let $M_{\ovo}$ be an ovo multi-classifier based on binary classifiers $C_{\{i,j\}}$. 
If the abstract multi-classifier {\rm $M_{\ovo}^\sharp$} defined in \eqref{eq-multiclass} 
is based on sound abstract binary classifiers $C^\sharp_{\{i,j\}}$ then {\rm $M_{\ovo}^\sharp$}
is sound for $M_{\ovo}$. 
\end{lemma}
\begin{proof}
Let $a\in A$, $\vec{x}\in \gamma(a)$, and let us prove that $M_{\ovo}(\vec{x}) \subseteq 
M_{\ovo}^\sharp(a)$. By soundness of all the binary classifiers 
$C^\sharp_{\{i,j\}}$, we have that for all $k\in [1,m]$,
if $\votes^\sharp(a,y_k) = [v^{\min}_k, v^{\max}_k]$ then
$v^{\min}_k = |\{i \in [1,m] \mid i\neq k, C^\sharp_{\{i,k\}}(a) =y_k\}|\leq 
|\{i \in [1,m] \mid i\neq k, C_{\{i,k\}}(\vec{x}) =y_k\}|$ 
and  $v^{\max}_k = |\{i \in [1,m] \mid i\neq k, C^\sharp_{\{i,k\}}(a) \in \{y_k,\top\}\}|
\geq |\{i \in [1,m] \mid i\neq k, C_{\{i,k\}}(\vec{x}) =y_k\}|$.  
Thus, for all $i\in [1,m]$, 
$v^{\min}_i \leq \votes(\vec{x},y_i) \leq v^{\max}_i$.
Consequently, 
if $y_i\in M_{\ovo}(\vec{x})$ then $i$ is a maximum
argument for $\votes(\vec{x},y_{(\cdot)})$, so that for all $j\neq i$,  
$\votes(\vec{x},y_j) \leq \votes(\vec{x},y_i)$. Hence, for all $j\neq i$, 
$v^{\min}_j \leq \votes(\vec{x},y_j) \leq \votes(\vec{x},y_i) \leq v^{\max}_i$, 
thus meaning that $y_i \in M_{\ovo}^\sharp(a)$. 
\qed 
\end{proof}

In our experimental evaluation we will follow the ovo approach for concrete multi-classification,
which is standard for SVMs~\cite{multiclass}, and consequently we will use this
abstract ovo multi-classifier in \eqref{eq-multiclass} 
for robustness verification.

\subsection{On the Completeness} \label{complete-sec}
Let $C:X \ra \{\minust, \plust\}$ be a binary classifier, $P:X\ra \wp(X)$ a perturbation 
and $C^\sharp:A\ra 
\{ \minust, \plust, \qmt\}$, $P^\sharp:X \ra A$ be a
corresponding sound abstract binary classifier and perturbation on an abstraction $A$. 
An abstract classifier can be used as a \emph{complete} robustness verifier when it is also able to prove unrobustness. 

\begin{definition}[Complete Abstract Classifiers and Robustness Verifiers]\label{def-complete}\rm
$C^\sharp$ is \emph{complete} for $C$ when for all $a\in A$, $C^\sharp(a)=\qmt \Ra 
\exists \vec{x},\vec{x}'\in \gamma(a).\, C(\vec{x})\neq C(\vec{x}')$. \\
$\tuple{C^\sharp,P^\sharp}$ is a (sound and) \emph{complete robustness verifier} for $C$ w.r.t.\ 
$P$ when
for all $\vec{x}\in X$, $C^\sharp(P^\sharp(\vec{x}))=C(\vec{x})$ iff $\Rob(C,\vec{x}, P)$.  
\qed
\end{definition}

Complete abstract classifiers can be obtained for linear binary
classifiers once these linear classifiers are in primal form and the abstract operations are exact.   
As discussed in Section~\ref{sec-arvf},  we therefore consider a linear binary classifier in primal form 
$C_{\prim} (\vec{x})\ud \sign_b(\sum_{j=1}^n \vec{w}_j \pi_j(\vec{x}))$ 
and an abstraction $A$ 
of $\wp(X)$ with $\gamma:A\ra \wp(X)$.  Let us consider the following exactness 
conditions for the abstract functions on $A$ needed for abstracting $C_{\prim}$ 
and the perturbation $P$: 
\begin{itemize}
\item[{\rm (E$_1$)}] Exact projection $\pi^\sharp_j$: 
For all $j\in [1,n]$ and $a\in A_n$, $\gamma(\pi^\sharp_j(a)) = \pi_j (\gamma(a))$;  
\item[{\rm (E$_2$)}] Exact scalar multiplication: For all $z\in \R$ and $a\in A_1$, 
$\gamma(z a)= z\gamma(a)$;  
\item[{\rm (E$_3$)}] Exact scalar addition $+^\sharp$: 
For all $a,a'\in A_1$, $\gamma(a+^\sharp a') = \gamma(a) + \gamma(a')$;  
\item[{\rm (E$_4$)}] Exact $\sign^\sharp_b$: For all  $b\in \R$, $a\in A_1$,  
$(\forall x\in \gamma(a). \sign_b(x)=s)$ $\Ra$ $\sign^\sharp_b(a)=s$;
\item[{\rm (E$_5$)}] Exact perturbation $P^\sharp$: For all  $\vec{x}\in X$, 
$\gamma(P^\sharp(\vec{x})) = P(\vec{x})$.
\end{itemize}
Then, it turns out that the abstract classifier 
\begin{equation}\label{complete-verif}
C_{\prim}^\sharp(a) \ud \sign^\sharp_b (\textstyle\sum_{j=1}^{\sharp n} \vec{w}_j \pi_j^\sharp(a))
\end{equation}
is complete and induces a complete robustness verifier.
\begin{lemma}\label{carv-lemma}
Under hypotheses {\rm (E$_1$)-(E$_5$)}, $C_{\prim}^\sharp$ in \eqref{complete-verif}
is (sound and) complete for $C_{\prim}$ 
and 
$\tuple{C_{\prim}^\sharp,P^\sharp}$ is a complete robustness verifier for $C_{\prim}$ w.r.t.\ $P$. 
\end{lemma}
\begin{proof}
Let $a\in A$  
and assume that 
$C_{\prim}^\sharp(a)=\sign^\sharp_b (\sum_{j=1}^{\sharp n} \vec{w}_j \pi_j^\sharp(a))=\qmt$. 
Then, by (E$_4$),  
we have that there exist $\alpha<0$ and $\beta\geq 0$ such that
$[\alpha,\beta]\subseteq \gamma(\sum_{i=1}^{\sharp n} \vec{w}_j \pi_j(a))$.
By (E$_1$), (E$_2$), (E$_3$), 
$\gamma(\sum_{i=1}^{\sharp n} \vec{w}_j \pi_j(a)) = 
\{\sum_{i=1}^{n} \vec{w}_j \pi_j(\vec{x}) \mid \vec{x} \in \gamma(a)\}$. Thus,
there exist $\vec{x},\vec{x}'\in \gamma(a)$ such that  
$\sum_{i=1}^{n} \vec{w}_j \pi_j(\vec{x})=\alpha$ and 
$\sum_{i=1}^{n} \vec{w}_j \pi_j(\vec{x}')$ $=$ $\beta$, implying that $C_{\prim}(\vec{x})=\minust$ and
$C_{\prim}(\vec{x}')=\plust$. Thus, $C_{\prim}^\sharp$ is complete.  

\noindent
Let us now prove in general that if $C^\sharp$ is sound and complete and $P^\sharp$ is exact then
$\tuple{C^\sharp,P^\sharp}$ is a complete robustness verifier.  
Let
$\vec{x}\in X$ and
assume that for all $\vec{x}'\in P(\vec{x})$, $C(\vec{x})=C(\vec{x}')$.  Then,
by (E$_5$), for all  $\vec{x}'\in \gamma(P^\sharp(\vec{x}))$, $C(\vec{x})=C(\vec{x}')$, 
so that, by completeness, $C^\sharp(P^\sharp(\vec{x})) \neq \qmt$. 
Since $P(\vec{x})=\gamma(P^\sharp(\vec{x})) \neq \varnothing$, by soundness of $C^\sharp$,
there is some $\vec{x}' \in P(\vec{x})$, such that $C^\sharp(P^\sharp(\vec{x}))=C(\vec{x}')=C(\vec{x})$.  
\qed
\end{proof}

Let us now focus on multi-classification. 
It turns out that completeness does not scale from binary to  multi-classification,  that is, 
even if all the abstract binary classifiers are assumed to be complete, the 
corresponding abstract multi-classification could lose the completeness. 
This loss is not due to the abstraction of the binary classifiers, but it is an intrinsic issue 
of a multi-classification approach based on binary classification. Let us show how this loss 
for ovr and ovo can happen through some examples. 

\begin{examplebf}\label{ex-ovr}\rm
Consider $L=\{y_1,y_2,y_3\}$ and assume that for some $\vec{x},\vec{x}'\in X$, the scoring ovr binary classifiers 
$C_{y_i,\overbar{y_i}}$
are as follows: 
$C_{y_1,\overbar{y_1}}(\vec{x}) = 3$, $C_{y_2,\overbar{y_2}}(\vec{x}) = -1$, $C_{y_3,\overbar{y_3}}(\vec{x}) = 2$,
$C_{y_1,\overbar{y_1}}(\vec{x}') = 1$,  $C_{y_2,\overbar{y_2}}(\vec{x}') = -1$, $C_{y_3,\overbar{y_3}}(\vec{x}) = 0.5$.
Hence,
$M_{\ovr}(\vec{x})=M_{\ovr}(\vec{x}')=\{y_1\}$, meaning that $M_{\ovr}$ is robust on $\vec{x}$ for a
perturbation $P(\vec{x})=\{\vec{x},\vec{x}'\}$. However, it turns out that the mere collecting abstraction
of binary classifiers $C_{y_i,\overbar{y_i}}$, although being trivially complete according to Definition~\ref{def-complete} 
may well lead to a (sound but) incomplete multi-classification. 
In fact, even if we consider no abstraction of sets of vectors/scalars and an abstract binary classifier is simply defined
by collecting abstraction
$C^\sharp_{y_i,\overbar{y_i}}(Y) \ud \{C_{y_i,\overbar{y_i}}(\vec{x}) \in \R \mid \vec{x}\in Y\}$, then we have that
while each $C^\sharp_{y_i,\overbar{y_i}}$ is complete the corresponding abstract ovr multi-classifier turns out to be 
sound but not complete. 
In our example, we have that: 
$C^\sharp_{y_1,\overbar{y_1}}(P(\vec{x})) = \{1,3\}$,  $C_{y_2,\overbar{y_2}}(P(\vec{x})) = \{-1,-1\}$, $C_{y_3,\overbar{y_3}}(P(\vec{x})) = \{0.5,2\}$.
Hence, the ovr strategy can only derive that both $y_1$ and $y_2$ are feasible classes
for $P(\vec{x})$, namely, $M^\sharp_{\ovr}(\{\vec{x},\vec{x}'\})=\{y_1,y_2\}$, meaning that  $M^\sharp_{\ovr}$
cannot prove the robustness of $M$. 
\qed
\end{examplebf}

The above example shows that the loss of relational information between input vectors and corresponding scores is an unavoidable 
source of incompleteness when abstracting ovr multi-classification. 
An analogous incompleteness happens in ovo multi-classification. 
\begin{examplebf}\rm
Consider $L=\{y_1,y_2,y_3,y_4,y_5\}$ and assume that for some $\vec{x},\vec{x}'\in X$, the ovo binary classifiers $C_{\{i,j\}}$
give the following outputs: 
\[
\begin{array}{|c|c|c|c|c|c|c|c|c|c|c|}
\hline & C_{\{1,2\}} & C_{\{1,3\}} & C_{\{1,4\}} & C_{\{1,5\}} & C_{\{2,3\}} & C_{\{2,4\}} & C_{\{2,5\}} & C_{\{3,4\}} & C_{\{3,5\}} & C_{\{4,5\}} \\[1pt]
\hline
\vec{x} & y_1&y_1 &y_1 &y_5 &y_2 &y_2 &y_5 &y_3 &y_3 &y_4 \\
\vec{x'}& y_1&y_1 &y_4 &y_1 &y_2 &y_4 &y_2 &y_3 &y_5 &y_5 \\
\hline
\end{array}
\]
%
so that $M_{\ovo}(\vec{x})=M_{\ovo}(\vec{x}')=\{y_1\}$, meaning that $M_{\ovo}$ is robust on $\vec{x}$
for the perturbation $P(\vec{x})=\{\vec{x},\vec{x}'\}$. Similarly to Example~\ref{ex-ovr}, the collecting abstractions
of binary classifiers $C_{y_i,y_j}$ are trivially complete but define a (sound but) incomplete multi-classification. 
In fact, even with no numerical abstraction, if we consider the abstract collecting binary classifiers 
$C^\sharp_{\{y_i,y_j\}}(Y) \ud \{C_{\{y_i,y_j\}}(\vec{x}) \mid \vec{x}\in Y\}$ then we have that:
$$
\begin{array}{|c|c|c|c|c|c|c|c|c|c|c|}
\hline & C_{\{1,2\}} & C_{\{1,3\}} & C_{\{1,4\}} & C_{\{1,5\}} & C_{\{2,3\}} & C_{\{2,4\}} & C_{\{2,5\}} & C_{\{3,4\}} & C_{\{3,5\}} & C_{\{4,5\}} \\[1pt]
\hline
\vspace*{1pt}
P(\vec{x}) & \{y_1\} &\{y_1\} &\{y_1,y_4\} &\{y_1,y_5\} &\{y_2\} &\{y_2,y_4\} &\{y_2,y_5\} &\{y_3\} &\{y_3,y_5\} &\{y_4,y_5\} \\
\hline
\end{array}
$$
%
Thus, the ovo voting for $P(\vec{x})$ in order to be sound necessarily has to assign $4$ votes to both classes $y_1$ and $y_5$,
meaning that $M^\sharp_{\ovo}(P(\vec{x}))=\{y_1,y_5\}$. As a consequence, $M^\sharp_{\ovr}$
cannot prove the robustness of $M_{\ovo}$. Here again, this is a consequence of the collecting abstraction which 
looses the relational information between input vectors and corresponding classes, and
therefore is an ineluctable source of incompleteness when abstracting ovo multi-classification. 
\qed
\end{examplebf}

Let us observe that when all the abstract binary classifiers $\ok{C^\sharp_{\{i,j\}}}$ are complete, then in
the abstract voting procedure \textsf{AV} in \eqref{def-av}, for all
$\votes^\sharp(a,y_i) = [v^{\min}_i,v^{\max}_i]$, we have that
$|\{j\neq i \mid \exists \vec{x}\in \gamma(a). C_{\{i,j\}}(\vec{x})=i\}| = v^{\max}_i$ holds, meaning that 
the hypothesis of completeness of abstract binary classifiers strengthens 
the upper bound $v^{\max}_i$ to a precise equality, although this is not enough for preserving the completeness. 

\section{Numerical Abstractions for Classifiers}\label{sec-numdom}

\subsection{Interval Abstraction} \label{int-subsec}
The $n$-dimensional 
interval abstraction  domain $\Int_n$ is simply defined as a nonrelational product of $\Int$, i.e., $\Int_n \ud \Int^n$
(with $\Int_1 = \Int$), where $\gamma_{\Int_n}: \Int_n \ra \wp(\R^n)$ is defined by $\gamma_{\Int_n}(I_1,...,I_n) \ud \times_{i=1}^n \gamma_{\Int}(I_i)$, and, by a slight abuse of notation, this concretization map will be denoted simply by $\gamma$.  
In order to abstract linear and nonlinear classifiers, we will use the following standard interval operations based on
real arithmetic operations.
\begin{itemize}
\item Projection $\pi_j:\Int_n \ra \Int$ defined by $\pi_j(I_1,...,I_n) \ud I_j$, which is trivially exact 
because $\Int_n$ is nonrelational.  
\item Scalar multiplication $\lambda \vec{I}.z \vec{I}: \Int_n \ra \Int_n$, with $z\in \R$, is defined as
componentwise extension of scalar
multiplication  $\lambda I.z I: \Int_1 \ra \Int_1$ given by: $z\bot = \bot$ and
$z[l,u] \ud [z l,zu]$,  where $z (\pm\infty)=\pm\infty$ for $z\neq 0$ and $0(\pm\infty) = 0$. This is 
an exact abstract scalar multiplication, i.e., $\{z\vec{x} \mid \vec{x}\in \gamma(\vec{I})\} = \gamma(z\vec{I})$ holds. 
\item Addition $+^\sharp: \Int_n \times \Int_n \ra \Int_n$ is defined as componentwise extension of standard interval addition, that is, 
$\bot +^\sharp I = \bot = I +^\sharp \bot$, $[l_1,u_1] +^\sharp [l_2,u_2] = [l_1+l_2,u_1 + u_2]$. 
This abstract interval addition  is exact, i.e., $\{\vec{x}_1 + \vec{x}_2 \mid \vec{x}_i\in \gamma(\vec{I}_i)\} = \gamma(\vec{I}_1 +^\sharp
\vec{I}_2)$ holds.  
\item One-dimensional 
multiplication $*^\sharp: \Int_1 \times \Int_1 \ra \Int_1$ is enough for our purposes, whose definition is standard:   
$\bot *^\sharp I = \bot = I *^\sharp \bot$, $[l_1,u_1] *^\sharp [l_2,u_2] = [\min(l_1 l_2, l_1 u_2, u_1 l_2,u_1 u_2),$ 
$\max(l_1 l_2, l_1 u_2, u_1 l_2,u_1 u_2)]$. 
As a consequence of the completeness of real numbers, 
this abstract interval multiplication is exact, i.e., $\{x_1 x_2 \mid x_i \in \gamma(I_i) \} 
= \gamma(I_1 *^\sharp I_2)$.  

\end{itemize}

It is worth remarking that since all these abstract functions on real intervals are exact and real intervals have the abstraction map, 
it turns out that all these abstract functions are the best correct approximations on intervals of the corresponding concrete functions. 

For the exponential function $e^{x}:\R \ra \R$ used by RBF kernels, let us consider a generic real function $f:\R \ra \R$ which
is assumed to be continuous and monotonically either 
increasing ($x\leq y\Ra f(x)\leq f(y)$) or decreasing ($x\leq y\Ra f(x)\geq f(y)$).
Its collecting
lifting $f^c:\wp(\R) \ra \wp(\R)$ is approximated on the interval abstraction by the abstract function
$f^\sharp: \Int_1 \ra \Int_1$
defined as follows: for all possibly unbounded intervals $[l,u]$ with $l,u \in \R\cup \{\text{--}\infty,\text{+}\infty\}$, 
\begin{align*}
&f_i([l,u]) \ud \inf \{f(x) \in \R \mid x\in \gamma([l,u])\}\in \R \cup \{\text{--}\infty\}\\
&f_s([l,u]) \ud \sup \{f(x) \in \R \mid x\in \gamma([l,u])\}\in \R\cup \{\text{+}\infty\}\\
&f^\sharp([l,u]) \ud [\min(f_i([l,u]), f_s([l,u])), \max(f_i([l,u]), f_s([l,u])] \qquad f^\sharp(\bot)\ud \bot
\end{align*} 
Thus, for bounded intervals $[l,u]$ with $l,u\in \R$, 
$f^\sharp([l,u]) = [\min(f(l),f(u)),$ $\max(f(l),f(u))]$.
As a consequence of the hypotheses of continuity and monotonicity of $f$, it
turns out that this abstract function $f^\sharp$ is exact, i.e.,  
$\{f(x)\in \R \mid x\in \gamma([l,u])\}=\gamma(f^\sharp([l,u]))$ holds, and
it is the best correct approximation on intervals of $f^c$.

\subsection{Reduced Affine Arithmetic Abstraction} \label{raf-subsec}
Even if all the abstract functions of the interval abstraction are exact, it is well known that 
the compositional abstract evaluation of an inductively defined expression $e$ on $\Int$
can be imprecise due to the dependency problem, meaning that if the syntactic expression $e$ includes multiple occurrences
of a variable $x$ and the abstract evaluation of $e$ is performed by structural induction on $e$, 
then each occurrence of $x$ in $e$ is taken independently from the others and this can lead to a significant loss of precision in the
output interval. This loss of precision may happen both for addition and multiplication of intervals. 
For example, the abstract compositional evaluations of the simple expressions $x-x$ and $x*x$ 
on an input interval $[-c,c]$, with $c\in \R_{>0}$, yield, respectively, $[-2c,2c]$ and $[-c^2,c^2]$, rather than
the exact results $[0,0]$ and $[0,c^2]$. This dependency problem can be a significant source of imprecision 
for the interval abstraction of a polynomial SVM classifier 
$C(\vec{x}) = \sign ( \textstyle [\sum_{i = 1}^N \alpha_i y_i (\sum_{j=1}^n (\vec{y_i})_j \vec{x}_j +c)^d] -b)$, 
where each attribute $\vec{x}_j$ of an input vector $\vec{x}$ occurs for each support vector $\vec{y_i}$. 
The classifiers based on RBF kernel show an analogous issue. 

\paragraph{\textbf{Affine Forms.}} 
Affine arithmetic \cite{sf97,sf2004} mitigates this dependency problem of the nonrelational interval abstraction. 
An interval $[l,u]\in \Int$ which approximates the range of some variable $x$ is 
represented by an \emph{affine form} (AF) $\hat{x} = a_0 + a_1 \epsilon_x$, where
$a_0 = (l+u)/2$, $a_1 = (u-l)/2$ and $\epsilon_x$ is a symbolic (or ``noise'') real variable ranging in $[-1,1]\in \Int$ which 
explicitly represents a dependence from the parameter $x$. This solves the dependency problem for a linear expression such as $x-x$
because the interval $[-c,c]$ for $x$ is represented 
by $0+c\epsilon_x$ so that the compositional evaluation of $x-x$ for $0+c\epsilon_x$ becomes $(0+c\epsilon_x) - (0+c\epsilon_x) = 0$, 
while for nonlinear expressions such as $x*x$, an approximation is still needed. 

In general, the domain $\AF_{k}$ of affine forms 
with  $k\geq 1$ noise variables consists of affine forms $\hat{a}=a_0 + \sum_{i=1}^k a_i\epsilon_i$, where
$a_i\in \R$ and each $\epsilon_i$ represents either an external dependence from some input variable or an internal approximation 
dependence due to a nonlinear operation. An affine form $\hat{a}\in \AF_{k}$ can be abstracted to
a real interval in $\Int$,  
as given by a map $\alpha_{\Int}:
\AF_{k}\ra \Int$ defined as follows: for all $\hat{a}=a_0 + \sum_{i=1}^k a_i\epsilon_i \in \AF_{k}$,
$\alpha_{\Int}(\hat{a}) \ud [c_{\hat{a}} - \rad(\hat{a}), c_{\hat{a}} + \rad(\hat{a})]$, where
$c_{\hat{a}}\ud a_0$ and 
$\rad(\hat{e}) \ud 
\sum_{i=1}^k |a_i|$ are called, resp., center and radius of the affine form $\hat{a}$. This, in turn, defines
the interval concretization $\gamma_{\AF_k}: \AF_k \ra \R$ given by
$\gamma_{\AF_k}(\hat{a}) \ud \gamma_{\Int\shortrightarrow \R}(\alpha_{\Int}(\hat{a}))$
$=\{x\in \R \mid c_{\hat{a}} - \rad(\hat{a}) \leq x \leq c_{\hat{a}} + \rad(\hat{a})\}$.

Vectors of affine forms may be used to represent zonotopes, which are center-symmetric convex polytopes and have been
used to design an abstract domain for static program analysis 
\cite{goubault2015} endowed with abstract functions, joins and widening. 

\paragraph{\textbf{Reduced Affine Forms.}} 
It turns out that affine forms are exact for linear operations, namely additions and scalar multiplications.
If $\hat{a},\hat{b}\in \AF_k$, 
where $\hat{a}=a_{0} + \sum_{j=1}^k a_j \epsilon_j$ and $\hat{b}=b_{0} + \sum_{j=1}^k b_j \epsilon_j$, 
and $c\in \R$ then abstract additions and scalar multiplications are defined as follows: 
\begin{align*}
\hat{a} +^\sharp \hat{b} \ud (a_0 + b_0) + \textstyle\sum_{j=1}^k  (a_j + b_j)\epsilon_j
\qquad\quad
c\hat{a} \ud c a_{0} + \textstyle\sum_{j=1}^k c a_j \epsilon_j.
\end{align*} 
These abstract operations are exact, namely,
$\{x + y\in \R \mid x\in \gammafk(\hat{a}), y\in 
\gammafk(\hat{b})\} = \gammafk(\hat{a} +^\sharp \hat{b})$ and $c\gammafk(\hat{a}) = \gammafk(c\hat{a})$.

For nonlinear operations, in particular multiplication, in general the result cannot be represented exactly by an affine form. 
Then, the standard strategy for defining the multiplication of affine forms is to approximate
the precise result by adding a fresh noise symbol  whose coefficient is typically computed by 
a Taylor or Chebyshev approximation of the nonlinear part of the multiplication  (cf.\ \cite[Section~2.1.5]{goubault2015}).  
Similarly, for the exponential function used in RBF kernels, an algorithm for computing an affine approximation of 
the exponential $e^x$ evaluated on 
an affine form $\hat{x}$ for the exponent $x$ is given in \cite[Section~3.11]{sf97} and is based on an optimal Chebyshev approximation 
(that is, w.r.t.\ $L_\infty$ distance) 
of the exponential which introduces a fresh noise symbol.  
However, the need of injecting a fresh noise symbol for each nonlinear operation 
raises a critical space and time  complexity issue for 
abstracting polynomial and RBF classifiers, because this would imply that a new but useless noise symbol should 
be added for each support vector.  For example, for a 2-polynomial classifier, we 
need to approximate a square operation $x*x$ for each of the $N$ support vectors,  
and a blind usage of abstract multiplication for affine forms 
would add $N$ different and useless noise symbols. This drawback would be even worse
for for $d$-polynomial classifers with $d>2$, while
an analogous critical issue would happen for RBF classifiers. 
This motivates the use of so-called \emph{reduced affine forms} (RAFs), which have been introduced in \cite{messine}
as a remedy for the increase of noise symbols due to nonlinear operations and still allow us to 
keep 
track of correlations between the components of the input vectors of classifiers. 
\\
\indent
A reduced affine form $\tilde{a}\in \RAF_k$ of length $k \geq 1$ is defined as a sum of a standard affine form 
in $\AF_k$ with 
a specific rounding noise $\epsilon_a$ which accumulates 
all the errors introduced by nonlinear operations. Thus, 
\begin{align*}
\RAF_k  \ud \{ 
a_0 + \textstyle\sum_{j=1}^k a_j \epsilon_j + a_r \epsilon_a \mid a_0,a_1,...,a_k \in \R,\, a_r\in \R_{\geq 0}\}.
\end{align*} 
The key point is that the length of $\tilde{a}\in \RAF_k$ remains unchanged during the whole abstract computation and 
$a_r \in \R_{\geq 0}$ is the radius of the accumulative error of approximating all nonlinear operations during
abstract computations. Of course, each $\tilde{a}\in \RAF_k$ can be viewed as a standard affine form in
$\AF_{k+1}$ and this allows us to define the interval concretization $\gamma_{\RAF_k}(\tilde{a})$
and the linear abstract operations
of addition and scalar multiplication of RAFs simply by considering them as standard affine forms. 
In particular, linear abstract operations in $\RAF_k$ are exact w.r.t.\ interval concretization $\gamma_{\RAF_k}$. 
\\
\indent
Nonlinear abstract operations, such as multiplication, must necessarily be approximated for RAFs. Several algorithms of
abstract multiplication of RAFs are available, which differ in precision, approximation principle and time complexity, 
ranging from linear to quadratic complexities \cite[Section~3]{skalna2017}. 
Given $\tilde{a},\tilde{b}\in 
\RAF_k$, we need to define an abstract multiplication 
$\tilde{a}*^\sharp \tilde{b}\in \RAF_k$ which is sound, namely,  
$\{xy \in \R \mid x\in \gamma_{\RAF_k}(\tilde{a}),\, y\in \gamma_{\RAF_k}(\tilde{b})\}$ $\subseteq$ $\gamma_{\RAF_k}(\tilde{a}*^\sharp \tilde{b})$, where it is worth pointing out that this soundness condition is given 
w.r.t.\ interval concretization $\gamma_{\RAF_k}$
and scalar multiplication. 
Time complexity is a crucial issue for using $*^\sharp$ in abstract polynomial and RBF kernels, because 
in these abstract classifiers at least an abstract multiplication must be used for each support vector, so that 
quadratic time algorithms in $O(k^2)$ cannot scale when the number of support vectors grows, as expected 
for realistic training datasets.
We therefore selected a recent linear time algorithm by Skalna and Hlad{\'{\i}}k \cite{skalna2017} which is 
optimal in the following sense.
Given $\tilde{a},\tilde{b}\in 
\RAF_k$, we have that their concrete symbolic multiplication is as follows:
\begin{align*}
\tilde{a} * \tilde{b} = &\; (a_0 + \textstyle\sum_{j=1}^k a_j \epsilon_j + a_r \epsilon_a) * 
(b_0 + \textstyle\sum_{j=1}^k b_j \epsilon_j + b_r \epsilon_b) \nonumber\\
= &\; a_0b_0 + \textstyle\sum_{j=1}^k(a_0b_j + b_0a_j)\epsilon_j + (a_0b_r\epsilon_b + b_0 a_r \epsilon_a) \: + f_{\tilde{a},\tilde{b}}(\epsilon_1,...,\epsilon_k,\epsilon_a,\epsilon_b)
\end{align*}
where $f_{\tilde{a},\tilde{b}}(\epsilon_1,...,\epsilon_k,\epsilon_a,\epsilon_b) \ud 
(\textstyle\sum_{j=1}^k a_j \epsilon_j + a_r\epsilon_a)(\textstyle\sum_{j=1}^k b_j \epsilon_j + b_r\epsilon_b)$.
An abstract multiplication $*_e^\sharp$ on $\RAF_k$ can be defined as follows:
if $R_{\max},R_{\min} \in \R$ are, resp., the minimum and maximum of 
$\{f_{\tilde{a},\tilde{b}}(\vec{e}) \in \R \mid \vec{e}\in [-1,1]^{k+2}\}$
then
\begin{align*}
\tilde{a} *_e^\sharp \tilde{b} \ud &\; a_0b_0 + 0.5(R_{\max}+R_{\min}) +
\textstyle\sum_{j=1}^k(a_0b_j + b_0a_j)\epsilon_j \: + 
(|a_0| b_r + |b_0|a_r + 0.5(R_{\max}-R_{\min})) \epsilon_{a*b}  
\end{align*}
where $0.5(R_{\max}+R_{\min})$ and $0.5(R_{\max}-R_{\min})$ are, resp., the center and the radius of 
the interval range of $f_{\tilde{a},\tilde{b}}(\epsilon_1,...,\epsilon_k,\epsilon_a,\epsilon_b)$. 
As argued in \cite[Proposition~3]{skalna2017}, 
this defines an optimal abstract multiplication of RAFs. 
Skalna and Hlad{\'{\i}}k \cite{skalna2017} put forward two algorithms for computing 
$R_{\max}$ and $R_{\min}$, one with $O(k)$ time bound 
and one in $O(k \log k)$: the $O(k)$ bound is obtained by relying on a linear time algorithm to find a median of a sequence of real
numbers, 
while the $O(k \log k)$ algorithm is based on (quick)sorting that sequence of numbers. The details of these algorithms are here omitted and
can be found in \cite[Section~4]{skalna2017}. 
In abstract interpretation terms, it turns out that this abstract multiplication algorithm $*_e^\sharp$ of RAFs provides 
the best approximation among the RAFs which correctly approximate the multiplication with the same coefficients for 
$\epsilon_1$,...,$\epsilon_k$ of $\tilde{a} *_e^\sharp \tilde{b}$:

Finally, let us consider the exponential function $e^x$ used in RBF kernels. The 
algorithm in \cite[Section~3.11]{sf97} for computing the affine form approximation of $e^x$ 
and based on Chebyshev approximation of $e^x$
can be also applied when the exponent is represented by a RAF 
$\tilde{x} = x_0 + \textstyle\sum_{j=1}^k x_j \epsilon_j + x_r \epsilon_x\in \RAF_k$, provided that the
radius of the fresh noise symbol produced by computing $e^{\tilde{x}}$ is added to the coefficient of the rounding noise 
$\epsilon_x$ of $\tilde{x}$.

\subsection{Floating Point Soundness}\label{sec-fpsound}
The interval abstraction, the reduced affine form 
abstraction and the corresponding 
abstract functions described in Section~\ref{sec-numdom} rely on precise real arithmetic on $\R$, in particular soundness and exactness
of abstract functions depend on real arithmetic. It turns out that these abstract functions may yield unsound results for 
floating point arithmetic such as the standard IEEE 754~\cite{ieee754}. This issue has been investigated in static program analysis 
by Min\'e \cite{mine04}
who showed how to adapt the interval abstraction and, more in general, a numerical abstract domain 
to IEEE 754-compliant floating point numbers in order to design ``floating point sound'' abstract functions. 
Let $\tuple{\F,\leq}$ denote the poset $\F$ of floating point numbers, where we refer to
the technical standard IEEE 754~\cite{ieee754}. 
The basic step 
is the rounding of a scalar $c\in \R$ towards $\text{--}\infty$, denoted by $c_{-}$, and towards $\text{+}\infty$, denoted by $c_{+}$, where
$c_{-} \!\in \F$ ($c_{+}\!\in \F$) is the greatest (least) floating-point number in $\tuple{\F,\leq}$ smaller (greater) than or equal to $c$.  
This floating-point rounding  allows to define  abstract operations on intervals of 
numbers in $\F$ which are floating-point sound \cite[Section~4]{mine04}. For example, if $[l_1,u_1], [l_2,u_2]\in \Int_{\F}$ are two intervals of floating-point numbers
then their floating point addition is defined by $[l_1,u_1] \ok{+^\sharp_f} [l_2,u_2] \ud [(l_1 + u_1)_{-},(l_2 + u_2)_{+}]$ and this operation is
floating-point sound because $\gamma([l_1,u_1] \ok{+^\sharp} [l_2,u_2]) \subseteq \gamma([l_1,u_1] \ok{+^\sharp_f} [l_2,u_2])$.  
A number of software packages for sound floating-point arithmetic are available, such as the GNU MPFR C library~\cite{mpfr} used
in the well-known library of numerical abstract domains Apron \cite{apron}. For efficiency reasons,  our prototype implementation 
simply uses the standard C  \textit{fesetround()} function for floating-point rounding \cite[Section~7.6]{C99} 
in all the definitions of our
floating-point sound abstract functions, both for intervals and RAFs.

\section{Verifying SVM Classifiers} \label{sec-vsvmc}

\subsection{Perturbations} \label{sec-pert}
We consider robustness of SVM classifiers against a standard adversarial region defined by the $L_\infty$ norm, 
as considered in Carlini and Wagner's robustness model \cite{carlini} and used by Vechev et al.~\cite{vechev-sp18,vechev-nips18,singh2019} 
in their verification framework. Given a generic 
classifier $C: X\ra L$ and a constant $\delta\in \R_{>0}$,  a $L_\infty$ $\delta$-perturbation
of an input vector $\vec{x}\in \R^n$ is defined by $P^\infty_\delta (\vec{x}) \ud \{\vec{x}'\in X \mid \norm{\vec{x}'-\vec{x}}_\infty \leq \delta\}$. 
Thus, if the space $X$ consists of $n$-dimensional real vectors normalized in $[0,1]$ (our datasets follow this standard)
 and $\delta \in (0,1]$ then 
\begin{align*}
P^\infty_\delta (\vec{x})=\{\vec{x}'\in \R^n \mid \forall i.\: \vec{x}'_i \in [\vec{x}_i-\epsilon,\vec{x}_i+\epsilon]\cap [0,1]\}.
\end{align*}  
Let us observe that, for all $\vec{x}$, $P^\infty_\delta (\vec{x})$ is an exact perturbation for intervals and therefore for RAFs as well (cf.\ (E$_5$)). The datasets of our experiments consist of $h\!\times\! w$ grayscale images where each image is represented as a real vector 
in $[0,1]^{hw}$ whose components encode the light values of pixels. Thus, increasing (decreasing) the value of a vector 
component means brightening (darkening) that pixel, so that a perturbation $P^\infty_\delta (\vec{x})$ 
of an image $\vec{x}$ represents all the images where every possible subset of pixels is brightened or darkened up to $\delta$.

We also consider robustness of image classifiers for the so-called adversarial framing on the border of images,
which has been recently shown to represent an effective attack for deep convolutional networks \cite{zajac}. 
Consider an image represented as a $h\times w$ matrix $M\in \R_{h,w}$ with normalized real values in $[0,1]$. Given an integer framing thickness
$t\in [1,\min(h,w)/2]$, the ``occlude'' $t$-framing perturbation of $M$ is defined by 
\begin{align*}
P^{\frm}_t (M) \ud \{M' \in 
\R_{h,w} \mid\;\: &\forall i\in [t+1,h-t], j\in [w+1,w-t].\, M'_{i,j}=M_{i,j}, \\ 
&\forall i\not\in [t+1,h-t], j\not\in [w+1,w-t].\, M'_{i,j}\in [0,1]\}.
\end{align*}
This framing perturbation models the uniformly distributed random noise attack in~\cite{zajac}.
Also in this case $P^{\frm}_t(M)$ is a perturbation which can be exactly represented by 
intervals and consequently by RAFs. 
An example of image taken from the Fashion-MNIST dataset together with 
two brightnening and darkening perturbations and a frame perturbation of thickness 2
is given in the figure below. 

\begin{center}
\begin{tabular}{cccc}
\includegraphics[scale=1.9]{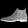}
&
\includegraphics[scale=1.9]{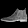}
&
\includegraphics[scale=1.9]{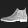}
&
\includegraphics[scale=1.9]{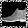}
\\[10pt]
{\footnotesize
{\; (1) Original image $\vec{m}$}}
&
{\footnotesize
{\; (2) Darkening in $P^\infty_{0.25}(\vec{m})$}}
&
{\footnotesize
{\; (3) Brightening in $P^\infty_{0.25}(\vec{m})$}}
&  
{\footnotesize
{\; (4) Framing in $P^{\frm}_2(\vec{m})$}}
\end{tabular}
\end{center}

\subsection{Linear Classifiers}
As observed in Section~\ref{sec-numdom}, for the interval abstraction 
it turns out that all the abstract functions
which are used in abstract linear binary classifiers in primal form (cf.\ Section~\ref{sec-arvf}) are exact, so that,
by Lemma~\ref{carv-lemma},
these abstract linear binary classifiers are complete. This completeness implies that there is no need to resort
to the RAF abstraction for linear binary classifiers. However, as argued in Section~\ref{sec-arvf}, 
this completeness for binary classifiers does not scale to multi-classification. 
Nevertheless, it is worth pointing out that for each binary classifier $C_{\{i,j\}}$ used in ovo multi-classification, 
since $L_\infty$ and frame perturbations are exact for intervals, we have a complete robustness verifier for each $C_{\{i,j\}}$. 
As a consequence, this makes feasible to find adversarial examples of linear binary classifiers as follows. 
Let us consider a linear binary classifier in primal form
$C(\vec{x}) = \sign ([\sum_{j = 1}^n \vec{w}_j \vec{x}_j] -b)$   
and a perturbation $P$ which is exact on intervals, i.e., for all $\vec{x}$, $P(\vec{x})=\gamma_{\Int_n}(P^\sharp(\vec{x}))$, where $P^\sharp(\vec{x}) = \tuple{[l_1,u_1],...,[l_n,u_n]})\in \Int_n$.   
Completeness of robustness linear verification means that if the interval abstraction 
$\textstyle\sum_{j=1}^{\sharp n} \vec{w}_j [l_j,u_j]$
outputs an interval $[l,u]\in \Int_1$ such that $0\in [l,u]$, then $C$ is  surely not robust on $\vec{x}$ for  $P$. 
It is then easy to find two input vectors $\vec{y},\vec{z} \in P (\vec{x})$ which provide 
a concrete counterexample to the robustness, namely such that 
$C(\vec{y})\neq C(\vec{z})$. 
For all $i\in [1,n]$, if we define: 
\begin{align*}
\vec{y}_i \ud \textbf{if}~\sign(\vec{w}_i)\geq 0~\textbf{then}~u_i~\textbf{else}~l_i 
\qquad\qquad
\vec{z}_i \ud  \textbf{if}~\sign(\vec{w}_i)\geq 0~\textbf{then}~l_i~\textbf{else}~u_i
\end{align*}
then we have found $\vec{y},\vec{z} \in P (\vec{x})$ such that 
$\sum_{j = 1}^n \vec{w}_j \vec{y}_j=u$ and $\sum_{j = 1}^n \vec{w}_j \vec{z}_j=l$,
so that 
$C(\vec{y})=\plust$ and $C(\vec{z})=\minust$. The pair of input vectors $(\vec{y},\vec{z})$ therefore represents 
the strongest adversarial example to the robustness of $C$ on $\vec{x}$. 

\subsection{Nonlinear Classifiers}
Let us first point out that interval and RAF abstractions are incomparable for nonlinear operations. 
\begin{examplebf}\label{ex-nonlinear}
\rm
Consider the $2$-polynomial in two variables $f(x_1,x_2)=(1+2x_1-x_2)^2 -\frac{1}{4}(2+x_1 +x_2)^2$,
which could be thought of as a 2-polynomial classifier in $\R^2$. Assume that $x_1$ and $x_2$ range in the interval $[-1,1]$. 
The abstract evaluation of $f$ on the intervals 
$I_{x_1}=[-1,1]=I_{x_2}$ is as follows:
\begin{align*}
f^\sharp_{\Int}(I_{x_1},I_{x_2}) &= (1+2[-1,1] -[-1,1])^2 -\textstyle\frac{1}{4} (2+[-1,1]+[-1,1])^2 \\
&= [-2,4]^2 -\textstyle\frac{1}{4} [0,4]^2 = [0,16]+[-4,0]=[-4,16]
\end{align*}
On the other hand, for the $\RAF_2$ abstraction we have that $\tilde{x}_1 = \epsilon_1$ and $\tilde{x}_2 = \epsilon_2$ and
the abstract evaluation of $f$ is as follows:
\begin{align*}
f^\sharp_{\RAF_2}(\tilde{x}_1, \tilde{x}_2) = &\; 
(1+2\epsilon_1 - \epsilon_2)^2 -\textstyle\frac{1}{4} (2+\epsilon_1 +\epsilon_2)^2 \\
= &\; [1 + 0.5(R_{1\max} + R_{1\min}) + 4\epsilon_1 -2\epsilon_2 + 0.5(R_{1\max} - R_{1\min})\epsilon_r] \, -\\
&\; \textstyle\frac{1}{4}[4 + 0.5(R_{2\max} + R_{2\min}) + 4\epsilon_1 +4\epsilon_2 + 0.5(R_{2\max} - R_{2\min})\epsilon_r]\\
{\normalsize \text{where}}\;\; &\; R_{1\max} =\max((2\epsilon_1-\epsilon_2)^2) = 9,\; R_{1\min} =\min((2\epsilon_1-\epsilon_2)^2)=0,\\
&\; R_{2\max} =\max((\epsilon_1+\epsilon_2)^2) = 4,\; R_{2\min} =\min((\epsilon_1+\epsilon_2)^2)=0\\
=&\; [5.5 +4\epsilon_1 -2\epsilon_2 + 4.5\epsilon_r] -\textstyle\frac{1}{4} [6 +4\epsilon_1 +4\epsilon_2 + 2\epsilon_r] \\
=&\; [5.5 +4\epsilon_1 -2\epsilon_2 + 4.5\epsilon_r] + [-1.5 -\epsilon_1 -\epsilon_2 -0.5\epsilon_r] \\
=&\; 4 +3\epsilon_1 -3\epsilon_2 +4 \epsilon_r
\end{align*}
Thus, it turns out that $\gamma_{\RAF_2}(f^\sharp_{\RAF_2}(\tilde{x}_1, \tilde{x}_2)) = [4-10,4+10]=[-6,14]$, which is 
incomparable
with $\gamma_{\Int}(f^\sharp_{\Int}(I_{x_1},I_{x_2}))=[-4,16]$. 
\qed
\end{examplebf}

By taking into account Example~\ref{ex-nonlinear}, for a nonlinear binary classifier 
$C(\vec{x}) = \sign(D(\vec{x}) - b)$, with $D(\vec{x})= \sum_{i = 1}^N \alpha_i y_i k(\vec{x}_i, \vec{x})$, 
we will use both the interval and RAF
abstractions of $C$  in order to combine 
their final abstract results. More precisely,  if $D^{\sharp}_{\Int_n}$ and $D^{\sharp}_{\RAF_n}$ are, resp., 
the interval and RAF abstractions of $D$, assume that $P:X\ra \wp(X)$ is a perturbation for $C$ which is soundly approximated 
by $P^{\sharp}_{\Int}: X \ra \Int_n$ on intervals and by $P^{\sharp}_{\RAF}: X \ra \RAF_n$ on RAFs, so that 
$P^\sharp: X \ra \Int_n \times \RAF_n$ is defined by $P^\sharp(x) \ud \tuple{P^{\sharp}_{\Int}(\vec{x}),P^{\sharp}_{\RAF}(\vec{x})}$. 
Then, for each input vector $\vec{x}\in X$, 
our combined verifier first will run both 
$D^{\sharp}_{\Int_n}(P^{\sharp}_{\Int}(\vec{x}))$ and $D^{\sharp}_{\RAF_n}(P^{\sharp}_{\RAF}(\vec{x}))$.
Next, the output 
$D^{\sharp}_{\RAF_n}(P^{\sharp}_{\RAF}(\vec{x})) = \hat{a}\in \RAF_n$ is 
abstracted to the interval $[c_{\hat{a}} - \rad(\hat{a}), c_{\hat{a}} + \rad(\hat{a})]$ which is then
intersected with the interval
$D^{\sharp}_{\Int_n}(P^{\sharp}_{\Int}(\vec{x})) = [l,u]$. Summing up, our combined abstract binary classifier $C^\sharp: \Int_n \times \RAF_n \ra \{\minust,\plust,\qmt\}$ is defined 
as follows: 
\begin{align*}
C^\sharp(P^\sharp(\vec{x})) \ud 
\begin{cases}
\plust & \text{if } \max(l,c_{\hat{a}} - \rad(\hat{a})) \geq 0\\
\minust & \text{if } \min(u,c_{\hat{a}} + \rad(\hat{a})) < 0\\
\qmt & \text{otherwise}
\end{cases}
\end{align*}
As shown in Section~\ref{sec-numdom}, it turns out that 
all the linear and nonlinear abstract operations for polynomial and RBF kernels are sound, so that by Lemma~\ref{srv-lemma}, 
this combined abstract classifier $C^\sharp$ induces a sound robustness verifer for $C$. 
Finally, for multi-classification,  
in both linear
and nonlinear cases, we will use the sound abstract ovo multi-classifier defined in Lemma~\ref{lemma-ovo}.

\section{Experimental Results}
\label{sec:experimental-evaluation}

We implemented our robustness verification method for SVM classifiers in
a tool called \emph{SAVer} (\emph{S}vm \emph{A}bstract \emph{Ver}ifier), which has been written in C (approximately 2.5k LOC) and is available open source in 
GitHub \cite{saver}
together with all the datasets, trained SVMs and results. 
We benchmark the percentage of samples of the full 
test sets for which a SVM classifier
is proved to be robust (and, dually, vulnerable) for a given perturbation, 
as well as the average verification times per sample. We also evaluated 
the impact of using subsets of the training set on the robustness of the corresponding classifiers and 
on verification times. We compared SAVer to DeepPoly \cite{singh2019}, a 
robustness verification tool for 
convolutional deep neural networks based on abstract interpretation.
Our experimental results indicate that SAVer is fast and scalable and that 
the percentage of robustness provable by SAVer for SVMs
is significantly higher than the robustness  provable by DeepPoly for deep neural networks. 
Our experiments for proving the robustness 
were run on a Intel Xeon E5520 2.27GHz CPU, while
the time measures used a AMD Ryzen 7 1700X 3.0GHz CPU.

\subsection{Datasets and Classifiers} 
For our experimental evaluation of SAVer we used the widespread and standard MNIST \cite{mnist} image dataset together and
the recent alternative Fashion-MNIST (F-MNIST) dataset \cite{fmnist}. They both contain grayscale images of $28 \!\times\! 28 \!=\! 784$ pixels
which are represented as normalized vectors of floating-point numbers in $[0,1]^{784}$. MNIST contains images of
handwritten digits, while F-MNIST comprises professional images of fashion dress products from 10 categories taken from the popular Zalando's e-com\-merce website. 
F-MNIST has been recently put forward as a more challenging 
alternative for the original MNIST dataset for benchmarking machine learning algorithms, since the extensive experimental results
reported in \cite{fmnist} showed that the test accuracy of
most machine learning classifiers significantly decreases (a rough average is about 10\%) from MNIST to F-MNIST.
In particular, \cite{fmnist} reports that the average test accuracy (on the whole test set) of
linear, polynomial and RBF SVMs on  MNIST is 95.4\% while for F-MNIST drops to 87.4\%, where RBF SVMs are reportedly the most precise classifiers on F-MNIST with an accuracy of  89.7\%. 
Both datasets include
a training set of 60000 images and a test set of 10000 images, with
no overlap. Our tests are run on the whole test set, where, following \cite{singh2019}, 
these 10000 images of MNIST and F-MNIST have 
been filtered out of those misclassified by the SVMs (ranging from ~3\% of RBF and
polynomial kernels to ~7\% for linear kernel), 
 while
the experiments comparing SAVer with DeepPoly are conducted on the same test subset.
We trained a number of SVM classifiers using different subsets of the training sets and different kernel functions.
We trained our SVMs with linear, RBF and (2, 3 and 9 degrees) polynomial kernels, and in order to benchmark the scalability 
of the verifiers we used the first 1k, 2k, 4k, 8k, 16k, 30k, 60k samples of the training set (training times never exceeded 3 hours). 
For training we used Scikit-learn \cite{scikit}, a popular 
machine learning library for Python, which relies on
the standard Libsvm C library \cite{libsvm} 
implementing the ovo approach for multi-classification. 

\subsection{Results} 
The results of our experimental evaluation are summarized by the charts~(a)-(h) and 
tables (a)-(c) below.  Charts~(a)-(f) refer to the dataset MNIST, while
charts~(g)-(h) compare MNIST and F-MNIST. 
Chart~(a) compares the provable robustness to a $P^\infty_\delta$ adversarial region  
of SVMs which have been 
trained with different kernels. 
It turns out that the 
RBF classifier is the most provably robust and is therefore taken as reference classifier for the successive charts. 
Chart~(b) compares the relative precision of robustness verification which can be obtained by changing the 
abstraction of the RBF classifier. As expected, the relational information of the RAF abstraction makes it 
significantly more precise than interval abstraction, although in some cases intervals can help in refining RAF analysis,
and this justifies their combined use.  

\medskip
\noindent
\hspace*{-40pt}
\begin{tabular}{cccc}
{\footnotesize (a)}
&
\includegraphics[align=c,scale=0.12]{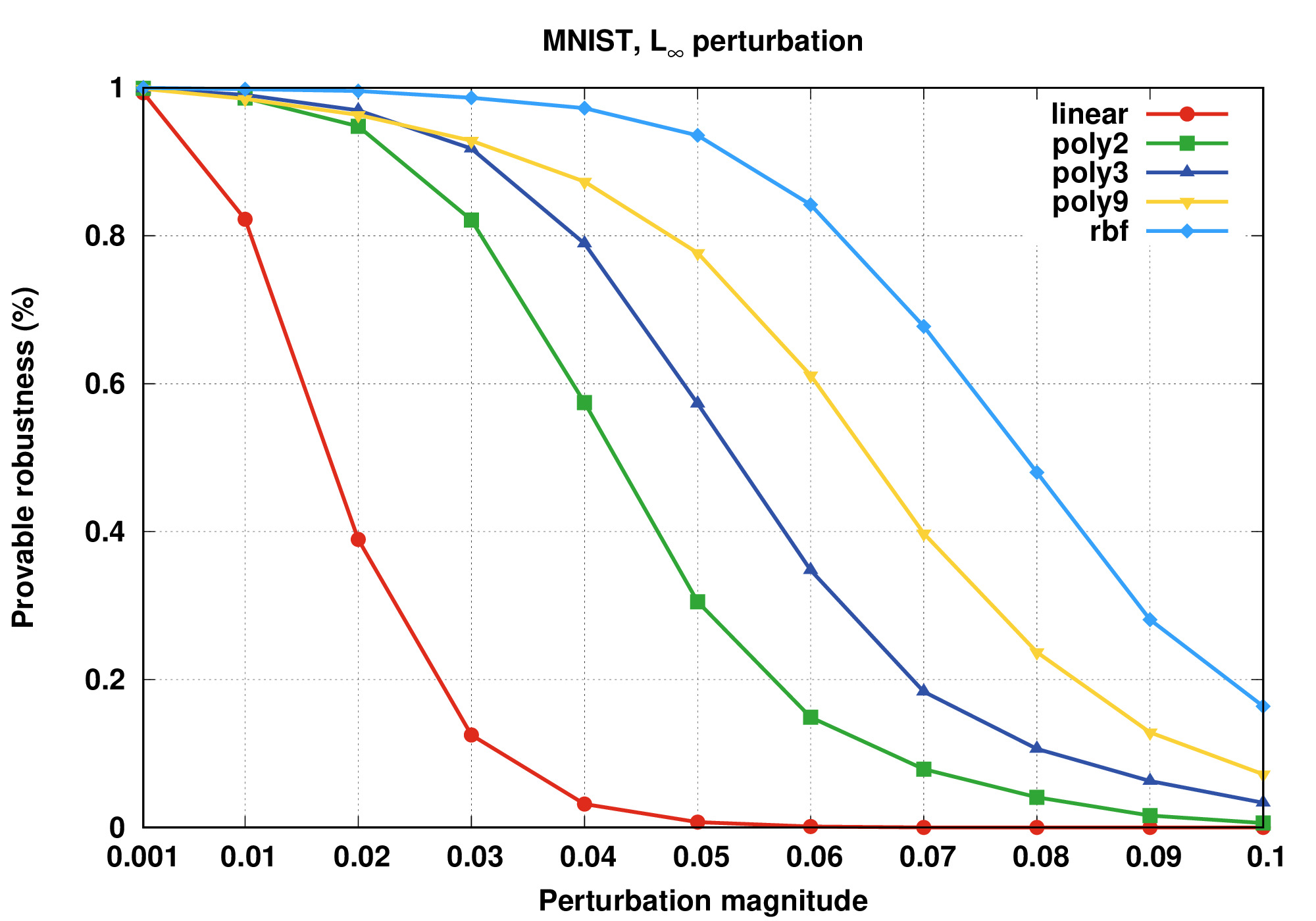}
&
{\footnotesize (b)}
&
\includegraphics[align=c,scale=0.12]{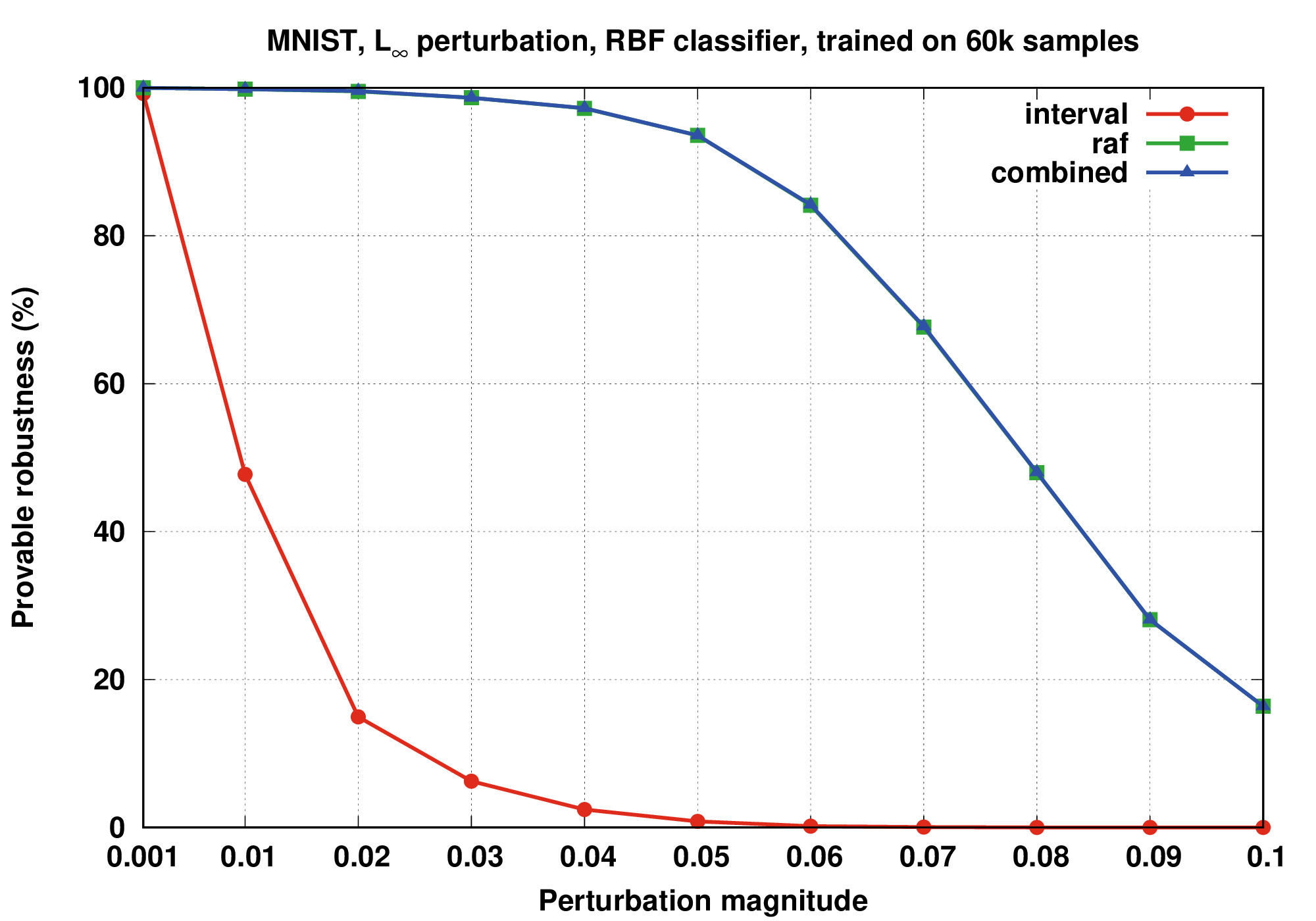}
\end{tabular}

\medskip
Chart~(c) shows how the provable robustness depends on the 
size of the training subset. We may observe here that using more samples for training a SVM classifier tends to overfit the model, making it more sensitive to perturbations, i.e.\ more vulnerable. Chart~(d) shows what we call \emph{provable 
vulnerability} of a classifier $C$: we first consider all the samples $(\vec{x},y)$ in the test set which are misclassified by $C$, i.e., $C(\vec{x})=y'\neq y$ holds,  
then our robustness verifier is run on the perturbations $P^\infty_\delta(\vec{x})$ of these samples for checking 
whether the region $P^\infty_\delta(\vec{x})$ can be proved to be consistently misclassified by $C$ to $y'$. 
Provable vulnerability is significantly lower than provable robustness, meaning that when the classifier is wrong on an input vector, 
it is more likely to assign different labels to similar inputs, rather than assigning the same (wrong) class.

\medskip
\noindent
\hspace*{-40pt}
\begin{tabular}{cccc}
{\footnotesize (c)}
&
\includegraphics[align=c,scale=0.120]{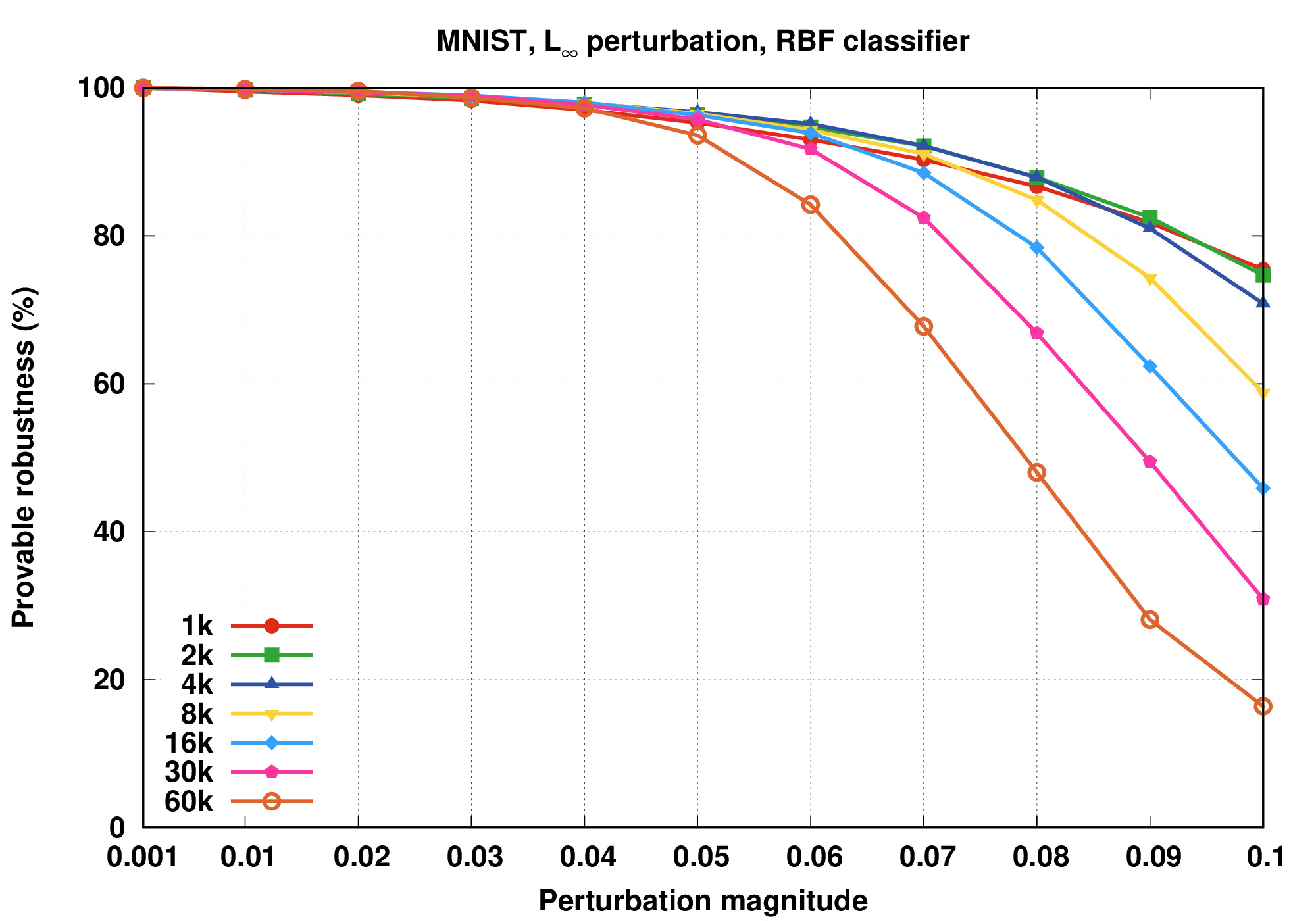}
&
{\footnotesize (d)}
&
\includegraphics[align=c,scale=0.120]{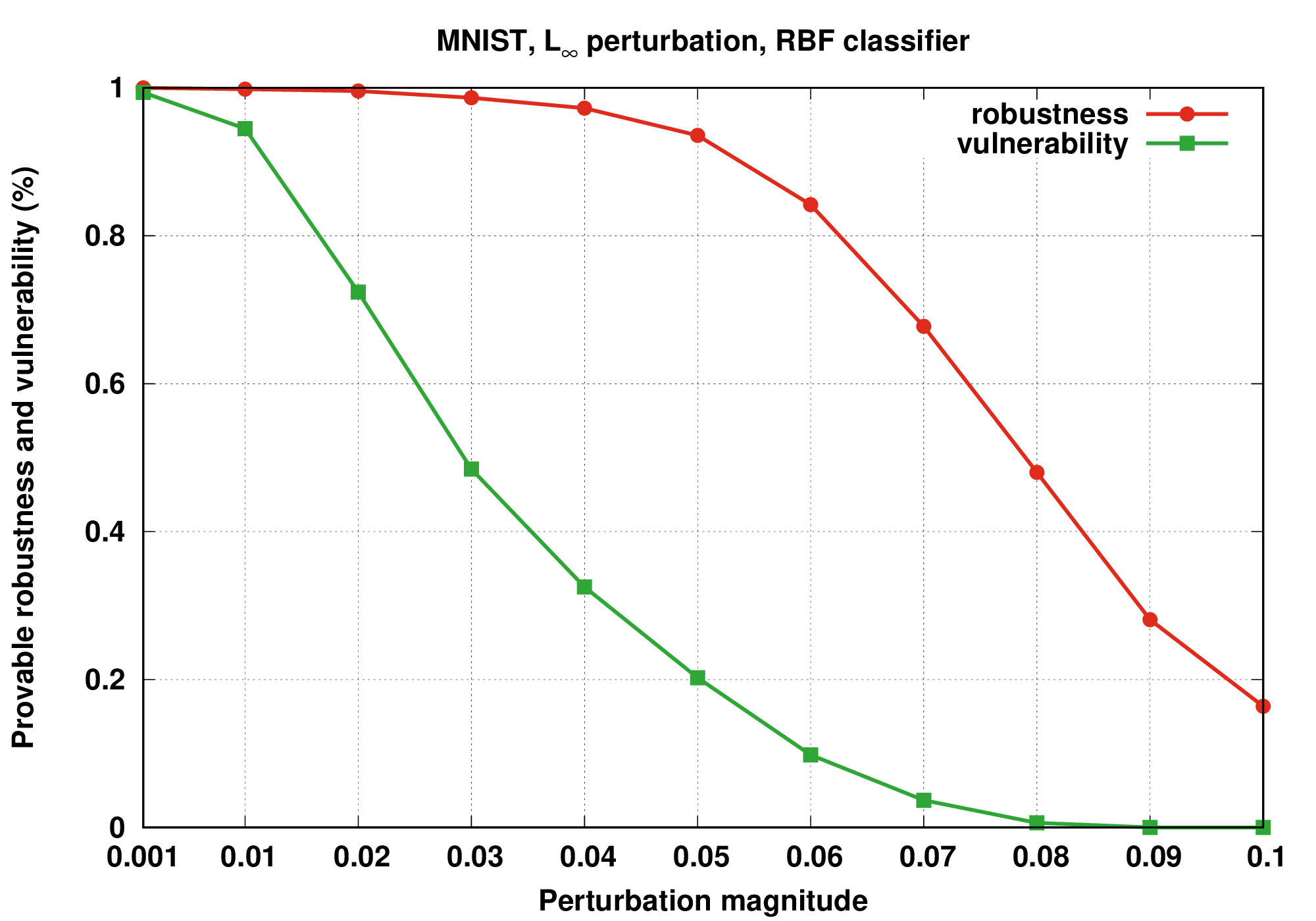}
\end{tabular}

\medskip
Charts~(e) and (f) show the average verification time per image, in milliseconds,
with respect to the size of the classifier, given by the number of
support vectors, and compared for different abstractions.  
Let $N$ and $n$ denote, resp., the number of support vectors and 
the size of input vectors. 
The interval-based abstract $d$-polynomial classifier is in $O(dN)$ time, while the RBF classifier is in $O(N)$, 
because the interval multiplication is constant-time. Hence, 
interval analysis is very fast, just a few milliseconds per image. On the other hand,  
the RAF-based abstract $d$-polynomial and RBF classifiers are, resp., in $O(d Nn\log n)$ and 
$O(N n \log n)$, since RAF multiplication is in $O(n\log n)$, so that RAF-based verification is slower although
it never takes more than 0.5 seconds. 
Table~(a) summarizes the precise 
percentages of provable robustness and average running times of SAVer for the RBF classifier.

\medskip
\noindent
\hspace*{-40pt}
\begin{tabular}{cccc}
{\footnotesize (e)}
&
\includegraphics[align=c,scale=0.120]{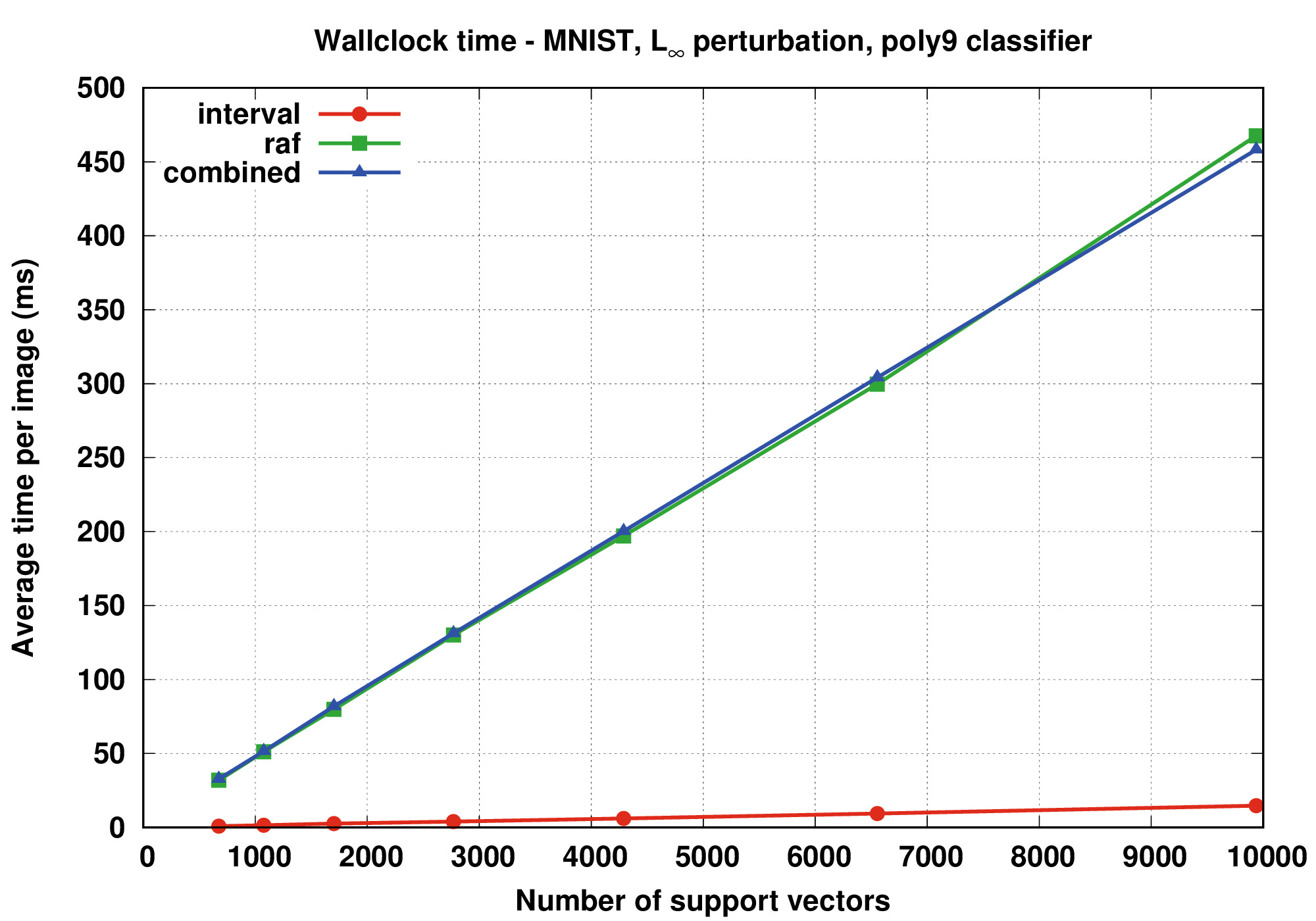}
& 
{\footnotesize (f)}
&
\includegraphics[align=c,scale=0.120]{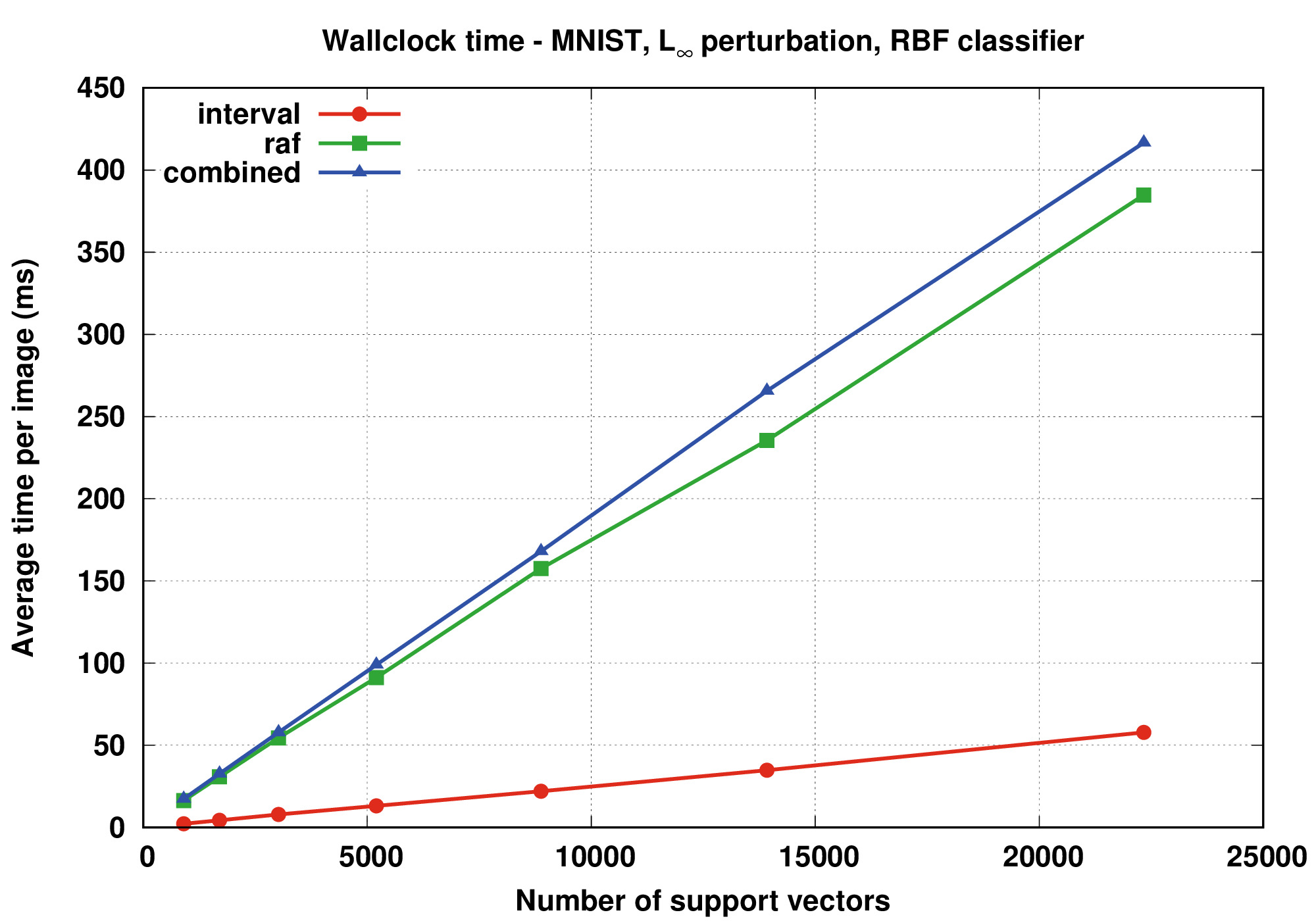}
\end{tabular}

\begin{center}
{\footnotesize (a)\;\;}%
{\footnotesize
 \begin{tabular}{ | c || c|c|c|c|c|c|c|c|c|c|c |}
  \hline
  {$P^\infty_\delta$} & 0.001 & 0.01 & 0.02 & 0.03 & 0.04 & 0.05 & 0.06 & 0.07 & 0.08 & 0.09 & 0.10\\
  \hline
  \bf MNIST: Provable & \multirow{2}{*} {100.00\%} & \multirow{2}{*}{99.83\%} & \multirow{2}{*}{99.57\%} & \multirow{2}{*}{99.19\%}  &  \multirow{2}{*}{97.27\%} & \multirow{2}{*}{93.58\%} & \multirow{2}{*}{82.21\%} &  \multirow{2}{*}{67.76\%} & \multirow{2}{*}{48.02\%} &  \multirow{2}{*}{28.10\%} & \multirow{2}{*}{16.38\%}\\
  \bf robustness for RBF &&&&&&&&&&&\\
  \hline
    \bf \;Time per image (ms)\; &
 416.49 & 417.18 & 415.95 & 417.19 & 416.98 & 417.69 & 417.21 & 416.93 & 417.21 & 417.15 
 & 417.97 \\
  \hline
 \end{tabular} 
}
\end{center}

\smallskip
The same experiments have been replicated on the F-MNIST dataset, and the Charts~(g) and (h) 
show a comparison of the results between MNIST and F-MNIST. 
As expected, robustness is harder to prove (and probably to achieve) 
on F-MNIST than on MNIST, while SAVer proved that F-MNIST is less vulnerable than MNIST. 
Moreover, Table~(b) shows the percentage of provable robustness for MNIST and F-MNIST for 
the frame adversarial region defined in Section~\ref{sec-vsvmc}, for some widths of the frame. 
F-MNIST is significantly harder to prove robust under this attack than MNIST: this is due 
to the fact that the borders of MNIST images do not contain as much information as their centers so that 
classifiers can tolerate some loss of information in the border. By contrast, 
F-MNIST images often carry information on their borders, making them more vulnerable to adversarial framing. 

\medskip
\noindent
\hspace*{-40pt}
\begin{tabular}{cccc}
{\footnotesize (g)}
&
\includegraphics[align=c,scale=0.120]{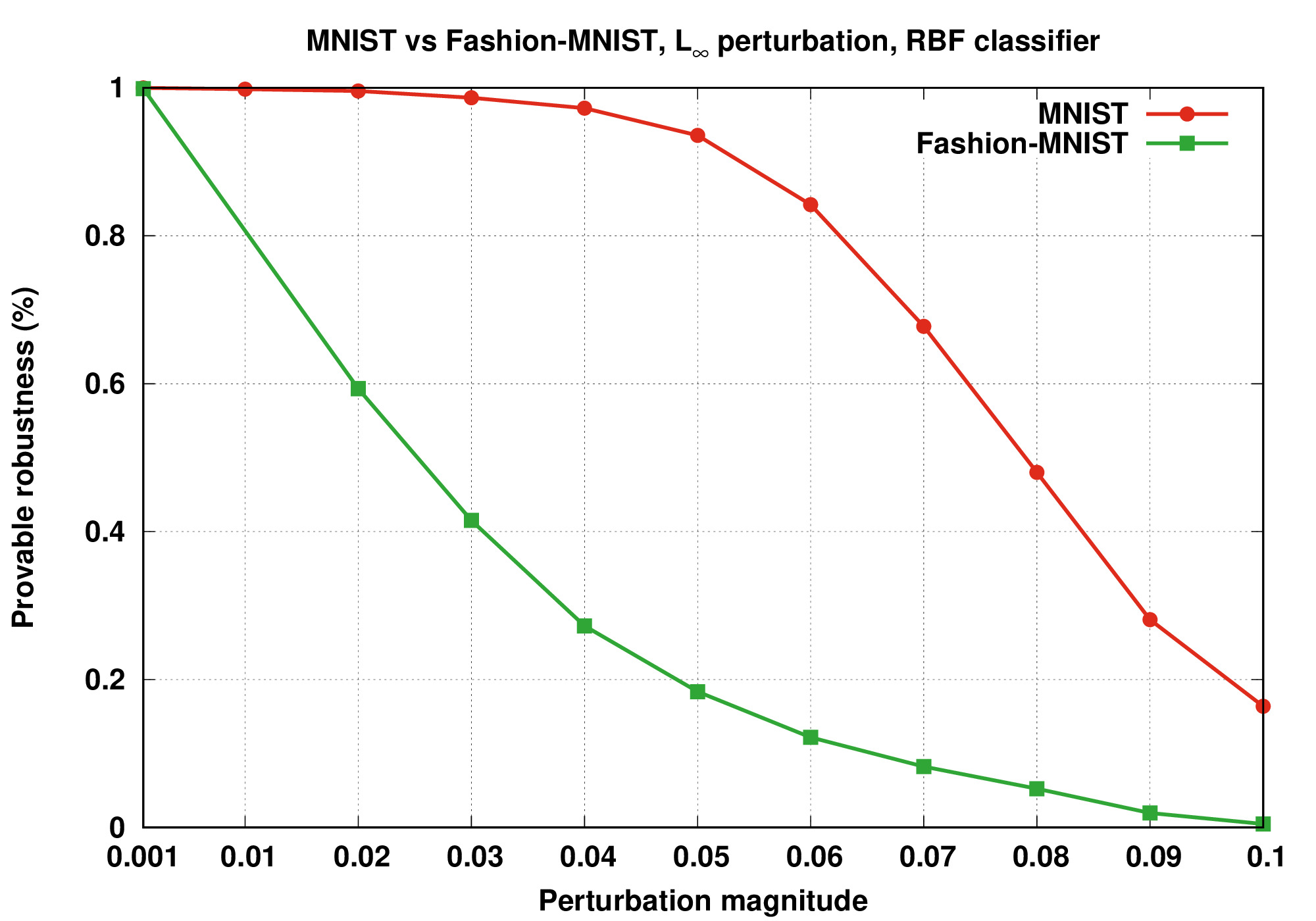}
& 
{\footnotesize (h)}
&
\includegraphics[align=c,scale=0.120]{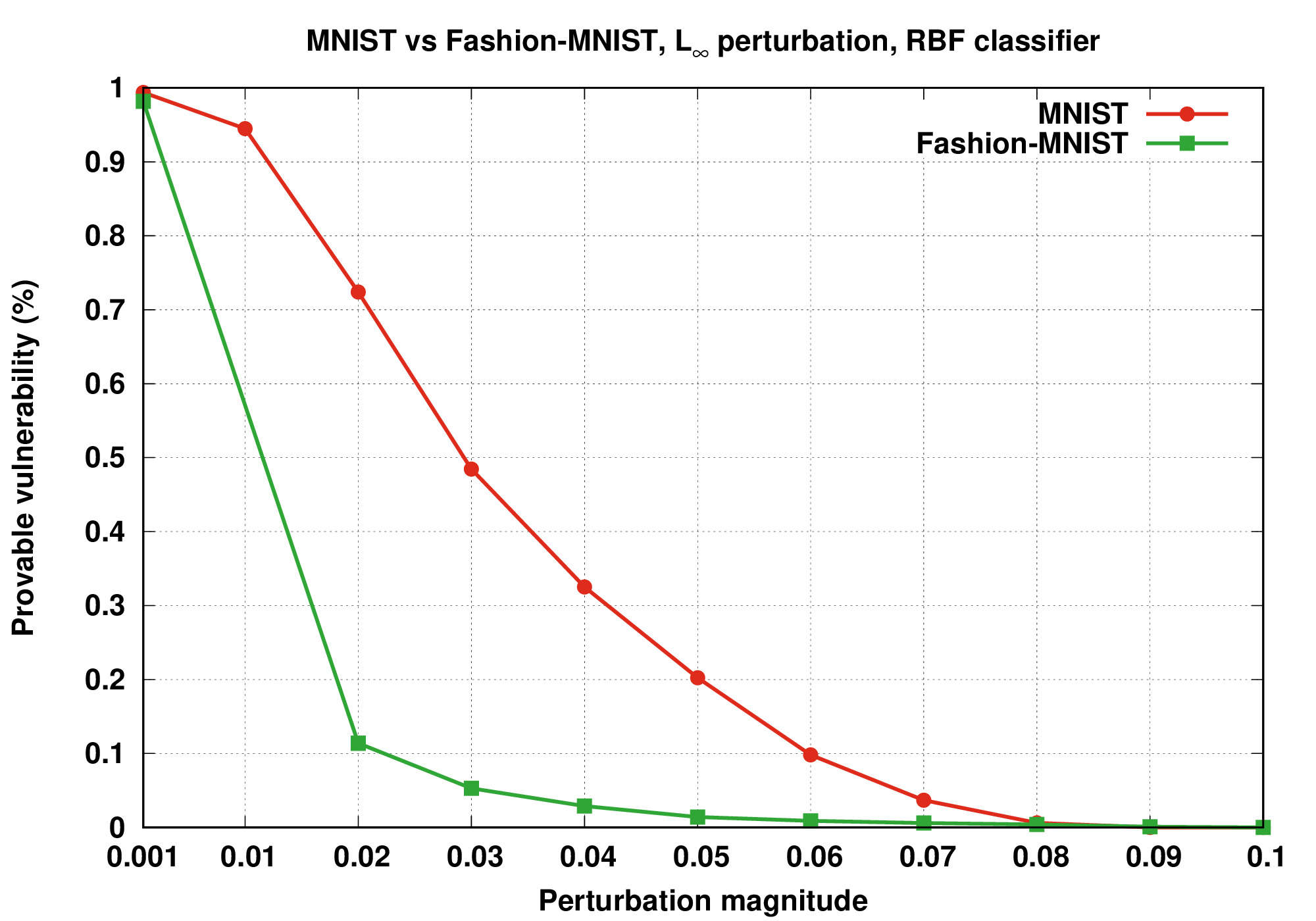}
\end{tabular}

\begin{center}
{\footnotesize (b)\;\;}%
{\footnotesize
 \begin{tabular}{ | c | c | c | }
  \hline
   \multirow{2}{*}{$P^{\frm}_t$} & {\bf MNIST: Provable} & {\bf F-MNIST: Provable} \\
  & {\bf Robustness} & {\bf Robustness} \\
  \hline
  1 & 100.00\% &  49.56\% \\
  2 &  99.64\% &   4.71\% \\
  3 &  87.34\% &   0.00\%  \\ 
  4 &  40.35\% &   0.00\% \\
  \hline 
 \end{tabular}
 }%
\end{center}

\medskip
Finally, Table~(c) compares SAVer with DeepPoly, a robustness verifier for feedforward neural networks \cite{singh2019}. 
This comparison used the same test set of DeepPoly, consisting of the first 100 images of the 
MNIST test set, and the same  perturbations $P^\infty_\delta$. Among
the benchmarks in \cite[Section~6]{singh2019}, we selected the FFNNSmall and FFNNSigmoid deep neural networks, denoted, resp.,
by DeepPoly Small and Sigmoid.  
FFNNSmall has been trained using a standard technique and achieved the best accuracies in \cite{singh2019}, while
FFNNSigmoid was trained using PGD-based adversarial training, a technique explicitly developed to make a classifier more robust.  

\smallskip
\begin{center}
 {\footnotesize (c)\;\;}%
 {\footnotesize
 \begin{tabular}{ | c | c | c | c | c | }
  \hline
  \multirow{2}{*}{$P^\infty_\delta$} & \bf SAVer  & \bf SAVer & \bf DeepPoly & \bf DeepPoly \\
   & \bf poly9 & \bf RBF & \bf Sigmoid & \bf Small \\
  \hline
  ~0.005 &  ~100\% & 100\% & 100\% & 100\% \\
  ~0.010 & ~98.9\% & 100\% &  98\% &  95\% \\
  ~0.015 & ~98.9\% & 100\% &  97\% &  75\% \\
  ~0.020 & ~97.8\% & 100\% &  95\% &  50\% \\
  ~0.025 & ~97.8\% & 100\% &  92\% &  25\% \\
  ~0.030 & ~96.7\% & 100\% &  80\% &  10\% \\
  \hline
 \end{tabular}
}
\end{center}

It turns out that the percentages of robustness provable by SAVer are significantly higher 
than those provable by DeepPoly (precise percentages are not provided in \cite{singh2019}, we extrapolated them from the charts). In particular, both 9-polynomial and RBF SVMs can be proved more robust that
FFNNSigmoid networks, despite the fact that these classifiers are defended by a specific adversarial training.

\section{Future Work}
We believe that this work represents a first step in applying formal
analysis and verification techniques to machine learning based on
support vector machines. We envisage a number of challenging research topics as subject for future work. Generating adversarial examples to machine learning 
methods is important for designing more robust classifiers \cite{advex,biggio15,zhao18} 
and we think that the completeness 
of robustness verification of linear binary classifiers (cf.\ Section~\ref{sec-arvf}) 
could be exploited 
for automatically detecting adversarial examples in linear multiclass SVM classifiers. The main challenge here is to design more precise, ideally complete, 
techniques for abstracting multi-classification based on binary classification. 
Adversarial SVM training is a further stimulating research challenge. Mirman et al.\ 
\cite{mirman2018differentiable} put forward an abstraction-based technique 
for adversarial training of robust neural networks. We think that a similar approach could also work for SVMs, namely
applying 
abstract interpretation to SVM training models rather than to 
SVM classifiers.  

\bibliographystyle{abbrv}
\bibliography{bibliography}

\end{document}